\documentclass{article}
\usepackage{iclr2027_conference, times}
\iclrfinalcopy
\usepackage[utf8]{inputenc} % allow utf-8 input
\usepackage[T1]{fontenc}    % use 8-bit T1 fonts
\usepackage{hyperref}       % hyperlinks
\usepackage{url}            % simple URL typesetting
\usepackage{amsfonts}       % blackboard math symbols
\usepackage{nicefrac}       % compact symbols for 1/2, etc.
\usepackage{xcolor}         % colors

\usepackage{booktabs}

\usepackage{caption}
\usepackage{subcaption}
\usepackage{amsmath}
\usepackage{amssymb}
\usepackage{mathtools}
\usepackage{amsthm}
\usepackage{thm-restate}

\usepackage{graphicx}

\usepackage[capitalize,noabbrev]{cleveref}

\theoremstyle{plain}
\newtheorem{theorem}{Theorem}[section]
\newtheorem{corollary}[theorem]{Corollary}

\theoremstyle{definition}
\newtheorem{definition}[theorem]{Definition}
\newtheorem{assumption}[theorem]{Assumption}

\theoremstyle{remark}

\title{Beyond Gaussian Worlds: Latent Geometry Matters for JEPAs}
\author{
Léo Nicollier$^{1,2}$ \quad
Enric Meinhardt-Llopis$^{1}$ \quad
Marc Pic$^{2}$ \quad
Pablo Musé$^{1,3}$ \quad
Gabriele Facciolo$^{1,4}$ \\[1ex]
$^{1}$Université Paris-Saclay, CNRS, ENS Paris-Saclay, Centre Borelli, France\\
$^{2}$Advanced Track and Trace\\
$^{3}$IIE, Facultad de Ingeniería, Universidad de la República, Uruguay\\
$^{4}$Institut Universitaire de France\\
\texttt{leo.nicollier@gmail.com}
}
\date{}

\begin{document}
\maketitle 
\lhead{Preprint} %a enlever pour soumission

\begin{abstract}
Recent Joint-Embedding Predictive Architectures (JEPAs) prevent representation collapse by constraining learned representations to follow a prescribed target distribution, such as an isotropic Gaussian or the uniform distribution on a hypersphere. 
Klindt et al. (2026) showed that, under their Euclidean assumptions, matching a Gaussian target can recover Gaussian latent variables up to a linear transformation, and that the Gaussian is the unique distribution with this guarantee.
We extend their analysis to latent variables supported on embedded Riemannian manifolds and derive conditions on the latent geometry and positive-pair dynamics under which alignment and exact distribution matching guarantee linear recovery.
In particular, when the latent variables are uniformly distributed on a sphere and the representations are matched to the same spherical distribution, every optimal representation recovers the latent state up to an orthogonal transformation.
This shows that Gaussian uniqueness is not a universal property of distribution-matched JEPAs: non-Euclidean latent geometries can admit other linearly recoverable distributions.
We further derive an approximate-recovery bound that is strictly tighter for the spherical world than for the Gaussian world.
Experiments on Gaussian, spherical, and toroidal latent spaces show that geometrically compatible targets yield better linear recovery when optimization succeeds, whereas mismatched targets distort the latent structure.
This advantage persists in high-dimensional Clifford-torus worlds.
%\gf{Finally,   in a simulator-access setting, we recover an unobserved initial condition of a three-dimensional two-body system from its observed trajectory.}{}
\end{abstract}

\section{Introduction}
\label{sec:introduction}

Joint-Embedding Predictive Architectures (JEPAs) learn representations by predicting or aligning related observations in a representation space rather than reconstructing them in the observation space~\citep{lecun2022path}.
To prevent collapse, recent methods constrain the marginal distribution of the learned representations.
LeJEPA targets an isotropic Gaussian distribution, whereas SPHERE-JEPA targets the uniform distribution on a hypersphere~\citep{balestriero2025lejepa,nicollier2026spherejepa}.
These choices are typically motivated by optimization or downstream performance, but they also implicitly impose a geometric prior on the learned representation space.

This raises a fundamental question of identifiability.
Training such a JEPA implicitly assumes that observations are generated from an underlying latent world through an unknown map \(g\) and that positive pairs encode proximity within this world; the identifiability question is whether the learned encoder \(f\) inverts this observation process, so that \(f\circ g\) recovers the latent state up to a linear transformation.
\citet{klindt2026doeslejepalearnworld} study this question in the Euclidean case \(\mathcal M=\mathbb R^d\), where positive pairs are formed by independently adding noise to each latent coordinate while preserving the latent distribution.
They show that, within this Euclidean setting, the Gaussian is the only latent distribution for which their framework guarantees linear recovery, and they also provide guarantees for approximate recovery.

Many latent variables, however, are non-Euclidean: rotations, directions, phases, and articulated configurations naturally lie on manifolds.
On such spaces, additive perturbations are not intrinsically defined, so the Euclidean formulation does not directly apply.
More broadly, when information about the geometry of the latent variables is available, it should guide the choice of representation target: a target with incompatible geometry can force the representation to distort the latent space.
Figure~\ref{fig:geometry-teaser} illustrates this effect.

%Figure~\ref{fig:geometry-teaser} illustrates this effect.
%When the target geometry matches the latent world, the learned representation preserves its global structure; when it does not, the global structure of the latent space can be distorted.

\begin{figure*}[t]
    \centering
    \setlength{\tabcolsep}{2pt}
    \renewcommand{\arraystretch}{0.82}
    \begin{tabular}{@{}lc|ccc@{}}
        &
        \scriptsize True latent space
        & \scriptsize Gaussian target
        & \scriptsize Spherical target
        & \scriptsize Torus target
        \\
        \rotatebox{90}{\scriptsize \hspace{3.5em}Gaussian world}
        &
        \includegraphics[width=0.22\textwidth]{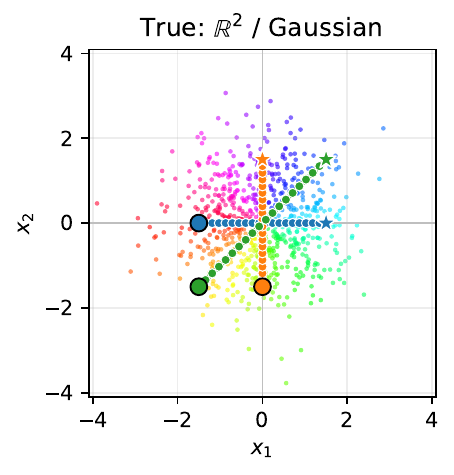}
        &
        \includegraphics[width=0.22\textwidth]{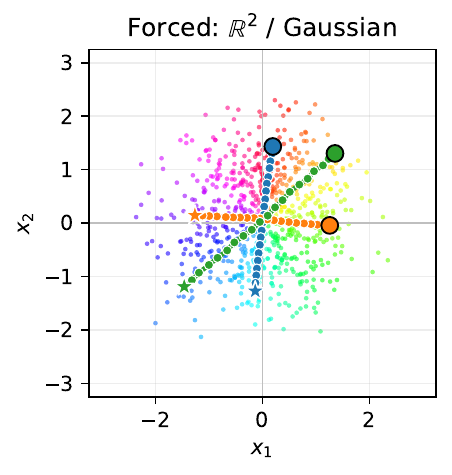}
        &
        \includegraphics[width=0.22\textwidth]{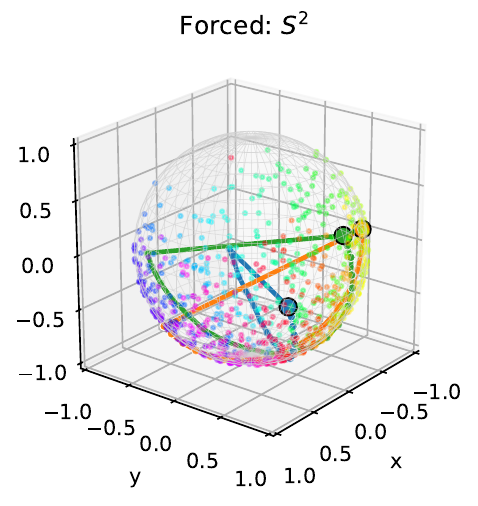}
        &
        \includegraphics[width=0.22\textwidth]{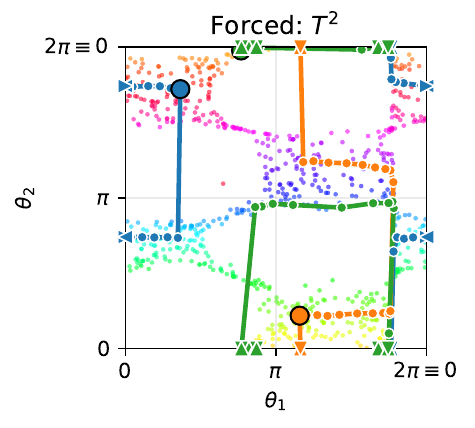}
        \\
        \rotatebox{90}{\scriptsize \hspace{3em}Spherical world}
        &
        \includegraphics[width=0.22\textwidth]{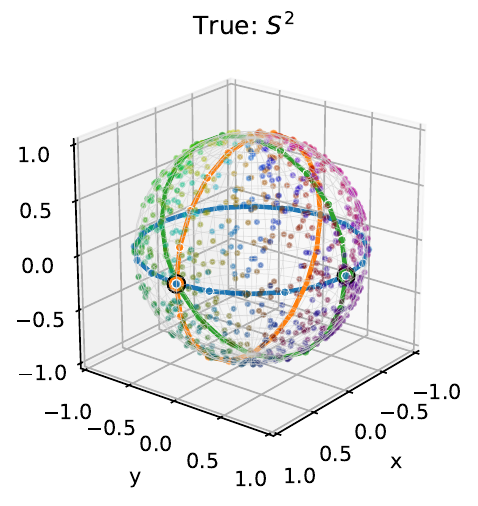}
        &
        \includegraphics[width=0.22\textwidth]{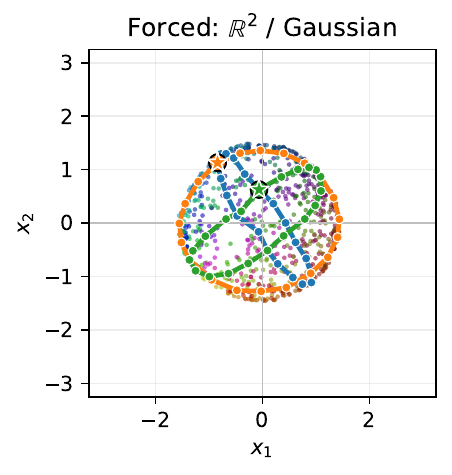}
        &
        \includegraphics[width=0.22\textwidth]{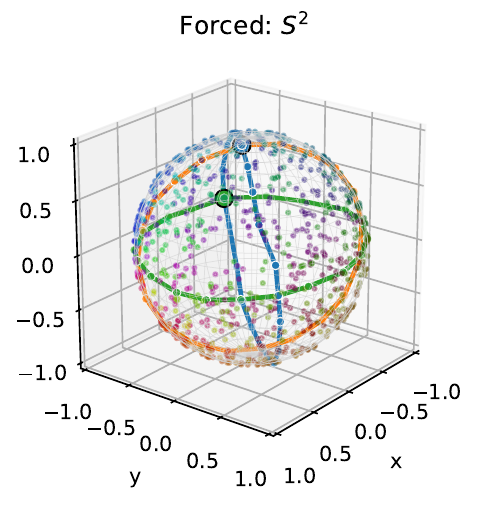}
        &
        \includegraphics[width=0.22\textwidth]{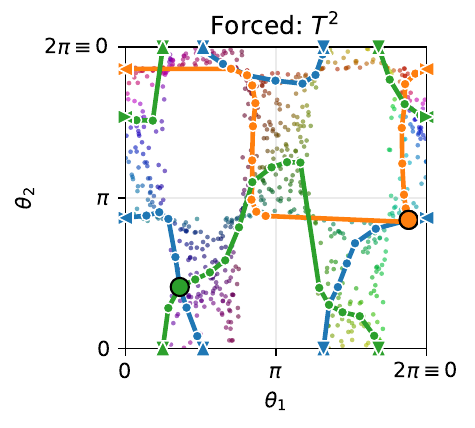}
        \\
        \rotatebox{90}{\scriptsize \hspace{3em}Toroidal world}
        &
        \includegraphics[width=0.22\textwidth]{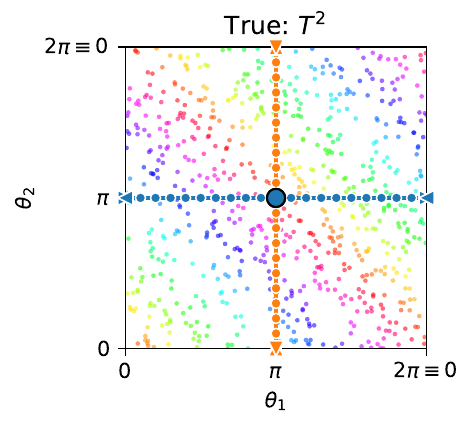}
        &
        \includegraphics[width=0.22\textwidth]{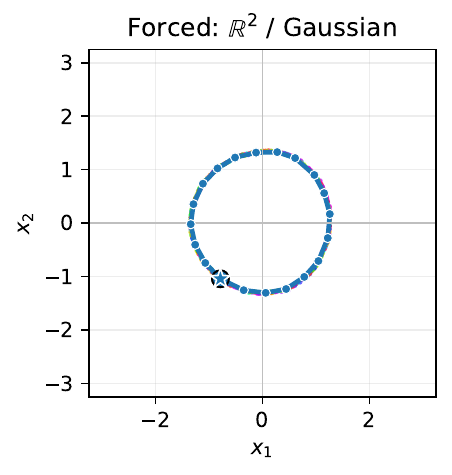}
        &
        \includegraphics[width=0.22\textwidth]{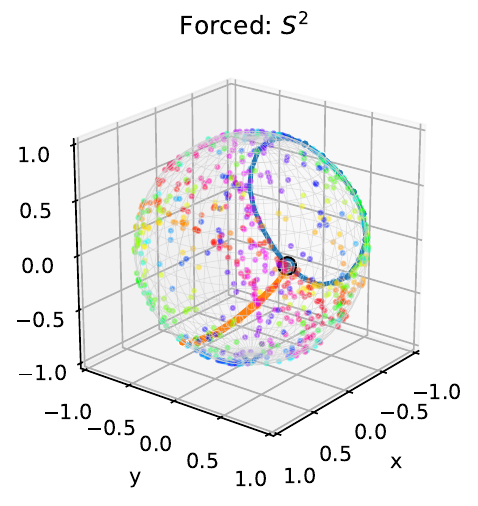}
        &
        \includegraphics[width=0.22\textwidth]{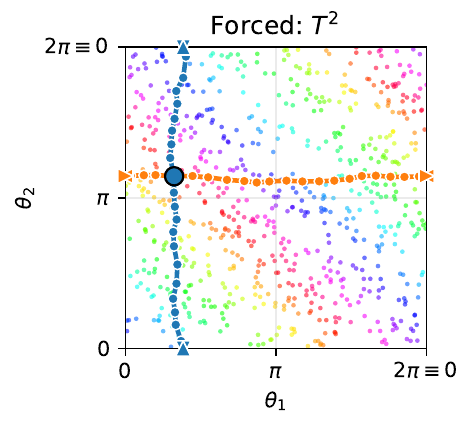}
    \end{tabular}
    \caption{
        Effect of target geometry on the learned representation.
        Rows correspond to three observation-generating systems: a Cartesian robot with a two-dimensional Gaussian latent position, an object with a spherical viewing direction, and a two-joint arm with a toroidal latent configuration.
        The first column shows the corresponding true latent space.
        In every other panel, a JEPA is trained on observations generated from the world in that row, with the representation constrained to match the target distribution indicated by the column.
        Colors and the green, blue, and orange paths track the same latent states across panels.
        For the three learned-representation columns, matched world--target pairs preserve the reference paths, whereas mismatched targets distort them.
    }
    \label{fig:geometry-teaser}
\end{figure*}

Building on \citet{klindt2026doeslejepalearnworld}, we extend identifiability theory from Euclidean latent spaces to distributions supported on embedded Riemannian manifolds, using intrinsic stochastic dynamics to generate positive pairs.
Our results show that Gaussian uniqueness does not extend universally to non-Euclidean latent spaces and that the representation target should reflect the geometry of the latent variables.

Our contributions are:
\begin{itemize}
    \item We derive conditions on manifold worlds under which alignment and exact distribution matching recover the latent state up to a linear transformation.
    Our results recover the Euclidean Gaussian world and identify the uniform sphere as a non-Euclidean identifiable world, with a tighter approximate-recovery guarantee than in the Gaussian case.
    \item We derive and implement geometry-aware heat-kernel MMD regularizers for the spherical, toroidal, and Clifford-torus target distributions used in our experiments.
    \item Experiments on Gaussian, spherical, toroidal, and high-dimensional Clifford-torus worlds show that compatible target geometries improve recovery when optimization succeeds.
    %\gf{We also recover hidden physical variables from observed dynamics in a simulator-access setting.}{}
\end{itemize}

\section{Related Work}
\label{sec:related_work}

\paragraph{Positioning.}
Distribution-matched JEPAs combine two sources of information: positive pairs specify which latent states should have similar representations, while the target distribution prevents collapse and prescribes the representation geometry.
Linear recovery depends on the compatibility of these two ingredients with the latent world.
We study this compatibility beyond Euclidean spaces using stationary manifold diffusions for pair generation, spectral analysis for identifiability, and geometry-aware kernels for practical distribution matching.

\paragraph{JEPAs and non-Euclidean representations.}
LeJEPA targets isotropic Gaussian representations, whereas SPHERE-JEPA targets the uniform distribution on a hypersphere~\citep{balestriero2025lejepa,nicollier2026spherejepa}.
More broadly, hyperspherical VAEs and manifold-valued latent-variable models incorporate known non-Euclidean structure into learned representations~\citep{davidson2018hyperspherical,falorsi2018homeomorphic}.
Unlike these generative approaches, we study when target-matched alignment identifies the underlying manifold-valued state.
Most closely related, \citet{klindt2026doeslejepalearnworld} establish Gaussian uniqueness for Euclidean worlds with stationary additive-noise transitions; we recover this Euclidean case while identifying non-Euclidean admissible worlds.

\paragraph{Identifiability and slow representations.}
Nonlinear Independent Component Analysis is generally unidentifiable from observations alone~\citep{hyvarinen1999nonlinear}.
Previous work restores identifiability by exploiting temporal structure or auxiliary variables~\citep{hyvarinen2016unsupervised,hyvarinen2019nonlinear}, conditional latent distributions~\citep{khemakhem2020variational}, or downstream predictive structure~\citep{roeder2021linear}.
Contrastive learning can likewise recover latent variables under assumptions on the data-generating process~\citep{zimmermann2021contrastive}.
In our setting, the additional information is provided by the positive-pair transition.
This connects alignment to Slow Feature Analysis~\citep{wiskott2002slow} and spectral views of self-supervised learning~\citep{balestriero2022contrastive}: alignment favors quantities that vary slowly across stationary Markov pairs, which are characterized by the leading nonconstant eigenfunctions of the transition operator.
We ask when these modes coincide with the ambient latent coordinates under exact distribution matching.

\paragraph{Manifold diffusions and spectral embeddings.}
The Euclidean additive-noise transitions studied by \citet{klindt2026doeslejepalearnworld} do not extend intrinsically to general manifolds.
We instead use stationary Langevin diffusions, which remain on the manifold and preserve its distribution~\citep{hsu2002stochastic,bakry2014analysis}.
Our spectral viewpoint is related to diffusion maps and Laplacian eigenmaps~\citep{nadler2005diffusion,belkin2003laplacian}, but our goal is different: we determine when a JEPA objective is forced to recover the ambient coordinates of the latent state rather than construct a new spectral embedding.

\paragraph{Kernel distribution matching.}
MMD and KSD provide kernel-based discrepancies for two-sample testing and goodness-of-fit, respectively~\citep{gretton2012kernel,liu2016kernelized}.
KerJEPA develops a general kernel-discrepancy framework for self-supervised learning, including MMD and KSD regularizers for Euclidean representations and hyperspherical uniformity~\citep{zimmermann2025kerjepa}.
Concurrently, \citet{nicollier2026expandingspherejepa} formulate full-dimensional MMD, KSD, and KL regularizers on the hypersphere and study heat and bandlimited spectral kernels.
Building on these developments, we use geometry-adapted heat-kernel MMDs for the Gaussian, spherical, and toroidal targets considered here; Appendix~\ref{app:spectral-mmd} gives their construction.

et quand ces\section{The World and the Learner}
\label{sec:world-learner}

Following the World--Learner formulation of \citet{klindt2026doeslejepalearnworld}, we separate the data-generating process from the learning objective.
The \emph{world} first samples a latent state \(Z\sim p\) on \(\mathcal M\).
Given \(Z=z\), it then samples a related state \(Z'\) from the conditional distribution \(K(z,\cdot)\), written \(Z'\mid Z=z\sim K(z,\cdot)\); thus, \(K\) determines how likely each state is to be selected as a positive partner.
The two states are converted into observations through an unknown map \(g\), which yields \((X,X')=(g(Z),g(Z'))\).
The \emph{learner} observes only \((X,X')\) and selects an encoder \(f\) that brings \(f(X)\) and \(f(X')\) close while matching the marginal distribution of \(f(X)\) to a prescribed target distribution.
This separation makes the identifiability question precise: For which worlds does every optimal induced latent map \(h=f\circ g\) recover \(Z\) up to a linear transformation?
We extend this formulation by allowing \(\mathcal M\) to be a Riemannian manifold and \(K\) to describe positive-pair dynamics intrinsic to \(\mathcal M\).
Figure~\ref{fig:world-learner-process} illustrates this pipeline for a two-joint robotic arm whose latent state lies on a torus.

\begin{figure*}[t]
    \centering
    \includegraphics[width=\textwidth]{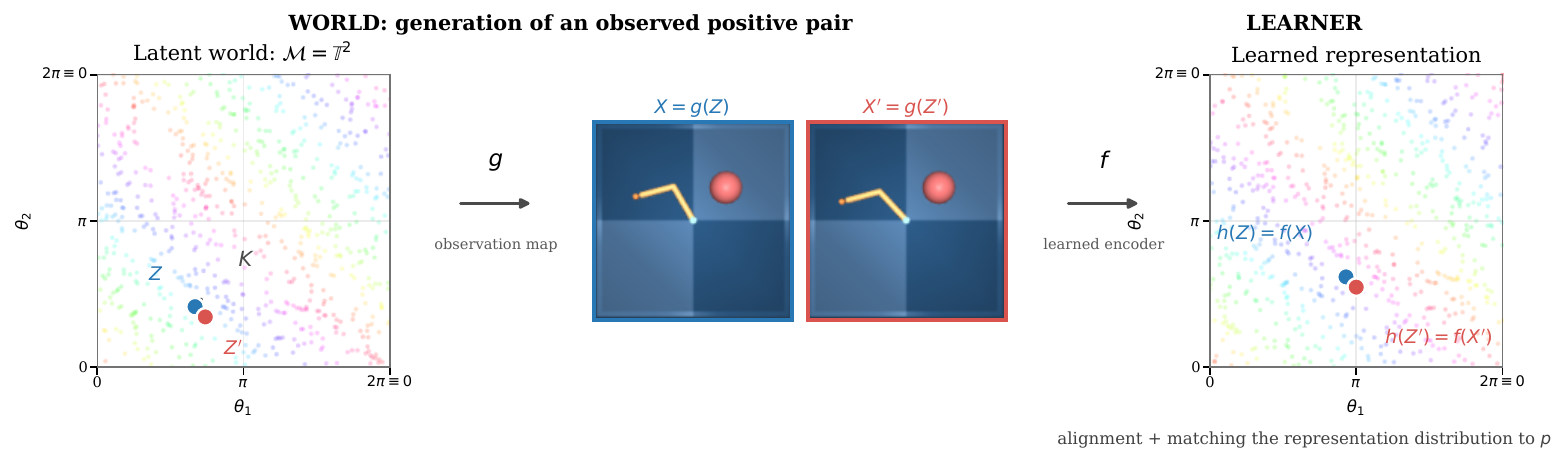}
    \caption{
        World--Learner framework illustrated on a two-joint robotic arm.
        The world samples a positive pair $(Z, Z')$ on the latent torus $\mathbb T^2$ according to $p=\operatorname{Unif}(\mathbb T^2)$ and $K$, then maps the two states to observations $(X, X')$ through the unknown observation map $g$.
        The learner observes only these images and trains an encoder $f$ to keep their representations close while matching their marginal distribution to the prescribed target.
        Blue denotes the initial state and red its positive partner throughout the pipeline.
    }
    \label{fig:world-learner-process}
\end{figure*}

\subsection{The World}
\label{sec:world}

Let $\mathcal X$ denote the observation space.

\begin{definition}[Latent world]
\label{def:latent-world}
A latent world is a tuple $\mathcal W=(\mathcal M,p,g,K)$, where $\mathcal M\subseteq\mathbb R^d$ is an embedded Riemannian manifold, $p$ is a probability measure on $\mathcal M$, $g:\mathcal M\to\mathcal X$ is a measurable observation map, and $K$ is a Markov transition kernel on $\mathcal M$.

An observed positive pair is generated according to
\[
Z\sim p,
\qquad
Z'\mid Z=z\sim K(z,\cdot),
\qquad
(X,X')=(g(Z),g(Z')).
\]
\end{definition}

To generate positive pairs intrinsically, we evolve $Z$ for a fixed time $t$ using a Langevin diffusion on $\mathcal M$.
For small $t$, this produces local perturbations that remain on the manifold while leaving $p$ invariant.
The following assumptions formalize this construction and ensure that the observations retain all information about the latent state.

\begin{assumption}[World]
\label{ass:world}
\; \vspace{-1.5em} \\
\begin{enumerate}
    \item \textbf{Positive density.}
    The measure \(p\) admits a smooth, strictly positive density with respect to the Riemannian volume on \(\mathcal M\).
    We use \(p\) to denote both the measure and its density when the meaning is clear.

    \item \textbf{Positive-pair dynamics.}
    For some fixed \(t>0\), the positive-pair kernel is \(K=K_t\), where \(K_t\) is the time-\(t\) transition kernel of the Langevin diffusion
    \begin{equation}
    \label{eq:sde}
        dZ_s
        =
        \nabla_{\mathcal M}\log p(Z_s)\,ds
        +
        \sqrt{2}\,dB_s^{\mathcal M},
    \end{equation}
    and \(B_s^{\mathcal M}\) is Brownian motion on \(\mathcal M\) with generator \(\frac12\Delta_{\mathcal M}\)~\citep{hsu2002stochastic}.

    \item \textbf{Observation recoverability.}
    The observation map \(g\) is injective \(p\)-almost surely and has a measurable inverse on its image.
    Equivalently, there exists a measurable map \(f_0:\mathcal X\to\mathcal M\) such that
    \[
        f_0\circ g=\operatorname{Id}_{\mathcal M}
        \qquad p\text{-almost surely}.
    \]
    Thus, although \(Z\) is not observed directly, no information about it is lost through \(g\).
\end{enumerate}
\end{assumption}

The infinitesimal generator of the diffusion in \eqref{eq:sde} is the weighted Laplacian
\begin{equation}
\label{eq:weighted-generator}
D_p\phi
=
\Delta_{\mathcal M}\phi
+
\left\langle
\nabla_{\mathcal M}\log p,
\nabla_{\mathcal M}\phi
\right\rangle_{\mathcal M}
=
\frac{1}{p}
\operatorname{div}_{\mathcal M}
\left(p\nabla_{\mathcal M}\phi\right).
\end{equation}
Its associated transition operator is
\[
(T_t\phi)(z)
\coloneqq
\mathbb E[\phi(Z_t)\mid Z_0=z]
=
\int_{\mathcal M}\phi(z')\,K_t(z,dz'),
\qquad
T_t=e^{tD_p}.
\]
The operators \(D_p\) and \(T_t=e^{tD_p}\) describe respectively the infinitesimal and time-\(t\) evolution of functions of the latent state and share the eigenfunctions underlying our analysis.

Under standard boundary or decay conditions, \(D_p\) is self-adjoint in \(L^2(p)\), and its Markov semigroup preserves \(p\)~\citep{bakry2014analysis}: \(\int_{\mathcal M}(T_t\phi)(z)\,p(dz)=\int_{\mathcal M}\phi(z)\,p(dz)\). Equivalently, \(\int_{\mathcal M}K_t(z,A)\,p(dz)=p(A)\) for every measurable \(A\subseteq\mathcal M\). Consequently, if \(Z\sim p\) and \(Z'\mid Z=z\sim K_t(z,\cdot)\), then \(Z'\sim p\); stationarity therefore follows from the positive-pair dynamics rather than constituting an additional assumption.

This framework includes the Euclidean Gaussian world, for which \eqref{eq:sde} is an Ornstein--Uhlenbeck process, and uniform compact worlds such as spheres and tori, for which the positive-pair dynamics reduce to intrinsic Brownian motion.
Canonical examples of latent worlds and their positive-pair dynamics are given in Appendix~\ref{app:canonical-worlds}.

\subsection{The Learner}
\label{sec:learner}

The learner is an encoder \(f:\mathcal X\to\mathcal M\), which induces a map between latent states and their representations,
\[
h=f\circ g:\mathcal M\to\mathcal M.
\]
Because the observation map \(g\) is recoverable, analyzing encoders on the observed data is equivalent, at the population level, to analyzing their induced latent maps \(h\).
Indeed, any measurable \(h\) can be realized on the support of the observations by setting \(f=h\circ g^{-1}\).
In particular, the identity map is feasible.

LeJEPA and SPHERE-JEPA combine two objectives: representations of positive pairs should remain close, and their marginal distribution should match a prescribed target distribution to prevent collapse~\citep{balestriero2025lejepa,nicollier2026spherejepa}.
Following the population formulation of \citet{klindt2026doeslejepalearnworld}, we idealize the second objective as exact distribution matching.
Our theory considers the matched setting, in which the representation space is \(\mathcal M\) and the target distribution is the latent distribution \(p\).
We consider the distribution-preserving alignment problem
\[
\min_{\substack{h:\mathcal M\to\mathcal M\\ h(Z)\sim p}}
\mathcal L_p(h),
\qquad
\mathcal L_p(h)\coloneqq
\mathbb E_{\substack{Z\sim p\\ Z'\mid Z\sim K(Z,\cdot)}}
\left[\|h(Z')-h(Z)\|^2\right].
\]

Here \(\|\cdot\|\) denotes the Euclidean norm in the ambient space \(\mathbb R^d\).
The objective brings the representations of positive latent states close.
The constraint requires the learned representations to have marginal distribution \(p\), so mapping every input to the same value is not allowed.
In the results below, we write this distribution-preservation condition compactly as \(h_{\#}p=p\).

\section{Main Results}
\label{sec:main-results}

The learning problem raises three questions.
Which latent worlds guarantee exact linear recovery?
How stable is this recovery when the alignment loss is only approximately optimal?
And which geometric properties strengthen or limit this stability?

%Our analysis relies on the spectrum of $-D_p$.
%Its eigenfunctions describe quantities that vary at different rates along the positive-pair dynamics: smaller eigenvalues correspond to functions that change more slowly and are therefore favored by the alignment objective.
%Linear recovery occurs when the slowest non-constant eigenfunctions are precisely the ambient coordinates of the latent state.
%Approximate recovery then depends on how well these linear modes are separated from the first competing nonlinear modes.

Intuitively, alignment favors functions of the latent state that vary most slowly along positive-pair transitions, while exact distribution matching rules out collapsed representations and fixes their marginal distribution.
Linear recovery is therefore possible when the ambient latent coordinates are exactly the complete set of slowest nonconstant functions.
The separation between these coordinate modes and the first nonlinear competitors controls stability: a larger spectral separation forces nonlinear representations to pay a larger alignment penalty.

The first two results give converse and forward conditions for exact linear recovery.
The third establishes stability under approximately optimal alignment while retaining exact distribution matching.
The final result shows how variation in the ambient radius ($\|z\|$) limits the separation between linear and nonlinear modes.

Throughout this section, let $\mathcal M\subseteq\mathbb R^d$ be a connected, properly embedded Riemannian manifold of intrinsic dimension $m$.
Let $p$ be a smooth, strictly positive stationary density, and let $K_t$ and $T_t=e^{tD_p}$ denote the positive-pair kernel and transition operator.
We assume that $-D_p$ is self-adjoint in $L^2(p)$ and has discrete spectrum.

We further assume that \(p\) is centered and isotropic:
\[
\mathbb E_{Z\sim p}[Z]=0,
\qquad
\operatorname{Cov}_{Z\sim p}(Z)=c_pI_d,
\]
for some \(c_p>0\).
These assumptions remove translation and anisotropic scaling from the remaining linear ambiguity.
Representations are required to preserve \(p\), written \(h_{\#}p=p\).

The following definition collects the three structural properties that will appear in both directions of our characterization.

\begin{definition}[Linearly admissible world]
\label{def:linearly-admissible}
For $\lambda>0$, we say that $(\mathcal M,p)$ is
\emph{linearly admissible at scale $\lambda$} if:

\begin{enumerate}
    \item \textbf{Density compatibility.}
    The density is the restriction to $\mathcal M$ of an isotropic Gaussian potential:
    \begin{equation}
    \label{eq:admissible-density}
        p(z)\propto
        \exp\!\left(-\frac{\lambda}{2}\|z\|^2\right),
        \qquad z\in\mathcal M.
    \end{equation}

    \item \textbf{Geometric compatibility.}
    The embedding satisfies
    \begin{equation}
    \label{eq:admissible-geometry}
        \Delta_{\mathcal M}z=-\lambda z^\perp,
    \end{equation}
    where $z^\perp$ is the component of the ambient position vector normal to $\mathcal M$.

    \item \textbf{Spectral ordering.}
    The ambient coordinate functions $z_1,\ldots,z_d$ span the complete first non-constant eigenspace of $-D_p$, with eigenvalue $\lambda$.
\end{enumerate}
\end{definition}

Together, the first two conditions ensure that each coordinate satisfies
$-D_pz_i=\lambda z_i$.
The third one says that these coordinates are the slowest non-constant functions of the latent state and that no nonlinear function varies equally slowly.
For a constant-radius manifold, such as a sphere or a Clifford torus, the density in \eqref{eq:admissible-density} is uniform.
Appendix~\ref{app:canonical-worlds} verifies these conditions for Gaussian spaces, spheres, equal-radius flat tori, and balanced products of spheres.

\paragraph{Necessary conditions.}
The spectral variational principle gives
$2dc_p(1-e^{-\lambda_1t})$ as a lower bound on the alignment loss.
Attaining this bound, or \emph{saturating} at it, means that an optimal representation uses only the slowest non-constant eigenfunctions.
The following result shows that if all such optimal representations are linear, then the world must be linearly admissible.

\begin{restatable}[Converse characterization]{theorem}{conversecharacterization}
\label{thm:converse-characterization}
Assume that the first non-constant eigenspace of $-D_p$ has eigenvalue $\lambda_1>0$ and
dimension $d$. Suppose that the distribution-preserving problem saturates the spectral lower bound,
\begin{equation}
\label{eq:spectral-tightness}
    \min_{\substack{h:\mathcal M\to\mathcal M\\ h_{\#}p=p}}\mathcal L_p(h)
    =2dc_p\left(1-e^{-\lambda_1t}\right),
\end{equation}
and that every distribution-preserving minimizer is linear in the ambient coordinates. Then
$(\mathcal M,p)$ is linearly admissible at scale $\lambda_1$.
\end{restatable}

Thus, under these assumptions, exact linear recovery forces compatibility between the latent density, the embedding geometry, and the slowest spectral modes.

\paragraph{Sufficient conditions.}
Linear admissibility is also sufficient: under the three conditions above, no nonlinear representation can attain the optimal alignment loss.

\begin{restatable}[Forward linear identifiability]{theorem}{forwardidentifiability}
\label{thm:forward-identifiability}
Assume that $(\mathcal M,p)$ is linearly admissible at scale $\lambda$. Then the minimizers of the distribution-preserving alignment problem are exactly the maps
\begin{equation}
\label{eq:forward-minimizers}
h^\star(z)=Qz,
\qquad
Q\in O(d),\quad Q(\mathcal M)=\mathcal M.
\end{equation}
Consequently, every optimal distribution-preserving representation recovers the latent state up to a linear symmetry of $(\mathcal M,p)$.
\end{restatable}

Here $O(d)$ denotes the orthogonal group, while $Q(\mathcal M)=\mathcal M$ restricts $Q$ to symmetries of the latent manifold.
Thus, the remaining ambiguity consists only of rotations or reflections that preserve the world.

\paragraph{Approximate recovery.}
Exact optimization is an idealization.
Let $\lambda$ denote the eigenvalue of the linear coordinate functions, and let $\lambda_{\mathrm{nl}}>\lambda$ denote the eigenvalue of the first competing nonlinear modes.
The corresponding quantities $\rho=e^{-\lambda t}$ and
$\rho_{\mathrm{nl}}=e^{-\lambda_{\mathrm{nl}}t}$ measure how strongly these modes remain correlated across positive pairs.
A larger separation $\rho-\rho_{\mathrm{nl}}$ makes linear and nonlinear representations easier to distinguish.

\begin{restatable}[Approximate linear identifiability]{theorem}{approximateidentifiability}
\label{thm:approximate-identifiability}
Assume that $(\mathcal M,p)$ is linearly admissible at scale $\lambda$, and let $\lambda_{\mathrm{nl}}>\lambda$ be the next distinct eigenvalue of $-D_p$. 
Set
\[
\rho=e^{-\lambda t},
\qquad
\rho_{\mathrm{nl}}=e^{-\lambda_{\mathrm{nl}}t}.
\]
Let $h:\mathcal M\to\mathcal M$ satisfy $h_{\#}p=p$ and suppose that
\begin{equation}
\label{eq:approximate-alignment-assumption}
    \mathcal L_p(h)
    \leq 2dc_p(1-\rho)+\varepsilon.
\end{equation}
Then there exists $A\in\mathbb R^{d\times d}$ such that
\begin{equation}
\label{eq:approximate-linear-bound}
    \|h-Az\|_{L^2(p)}^2
    \leq
    \frac{\varepsilon}{2(\rho-\rho_{\mathrm{nl}})}.
\end{equation}
Moreover, there exists $Q\in O(d)$ such that, with
$\eta=\varepsilon/[2(\rho-\rho_{\mathrm{nl}})]$,
\begin{equation}
\label{eq:approximate-orthogonal-bound}
    \|h-Qz\|_{L^2(p)}^2
    \leq \eta+\frac{\eta^2}{c_p}.
\end{equation}
The constant in \eqref{eq:approximate-linear-bound} is optimal at the
level of the underlying spectral inequality.
\end{restatable}

Here \(\varepsilon\) measures the excess alignment loss above the optimum under exact distribution preservation.
The first bound shows that a nearly optimal representation must be close to a linear map, with an error controlled by
\(\varepsilon/(\rho-\rho_{\mathrm{nl}})\).
The second strengthens this conclusion by showing closeness to an orthogonal transformation of the latent coordinates.
Appendix~\ref{app:approximate-bound-experiment} empirically examines this bound on learned Clifford-torus representations.

\paragraph{The role of the ambient radius.}
The final result identifies a universal nonlinear competitor.
When $\|z\|$ varies across the manifold, the centered squared radius is a non-constant function of the latent state and limits how far the first nonlinear eigenvalue can lie above the linear one.

\begin{restatable}[Radial obstruction and spherical rigidity]{theorem}{radialobstruction}
\label{thm:radial-obstruction}
Let $(\mathcal M^m,p)$ be linearly admissible at scale $\lambda$.
Then $r(z)=\|z\|^2-m/\lambda$ satisfies $-D_pr=2\lambda r$. 
Hence, either $\mathcal M\subseteq\mathbb S^{d-1}(\sqrt{m/\lambda})$, or $\lambda_{\mathrm{nl}}\leq2\lambda$, where $\lambda_{\mathrm{nl}}>\lambda$ denotes the next distinct eigenvalue of $-D_p$.
If moreover $m=d-1$ and $\mathcal M$ is closed, then $\lambda_{\mathrm{nl}}>2\lambda$ implies $\mathcal M=\mathbb S^{d-1}(\sqrt{(d-1)/\lambda})$.
\end{restatable}

If the ambient radius varies, the radial function produces a nonlinear eigenmode at eigenvalue $2\lambda$, implying $\lambda_{\mathrm{nl}}\leq2\lambda$.
On a constant-radius manifold this mode vanishes.
Moreover, among closed hypersurfaces, the sphere is the only linearly admissible geometry whose first nonlinear eigenvalue can lie strictly above $2\lambda$.

\begin{corollary}[Gaussian versus spherical separation]
For the Euclidean Gaussian world,
$\lambda_{\mathrm{nl}}/\lambda=2$, whereas for the uniform spherical world,
$\lambda_{\mathrm{nl}}/\lambda=2d/(d-1)>2$. Thus, the spherical world has a strictly larger relative linear--nonlinear spectral separation.
\end{corollary}

The comparison is therefore concrete: the first nonlinear competitor appears exactly at twice the linear eigenvalue in the Gaussian world, but later in the spherical world.
At fixed linear-mode correlation \(\rho\), this larger separation yields a tighter approximate-recovery guarantee for the sphere.

\section{Why the Target Geometry Matters}
\label{sec:target-geometry}

A Gaussian target is not geometrically neutral.
Although an isotropic Gaussian concentrates near a sphere as dimension increases, its radius remains variable at every finite dimension.
Consequently, the first nonlinear mode appears at twice the linear eigenvalue, so that \(\lambda_{\mathrm{nl}}/\lambda=2\).
On the uniform sphere, the radius is constant and
\(\frac{\lambda_{\mathrm{nl}}}{\lambda}=2+\frac{2}{d-1}>2\).
At fixed linear-mode correlation \(\rho\), the sphere therefore yields a tighter approximate-recovery bound, although this relative advantage vanishes as \(d\to\infty\); Appendix~\ref{app:spectral-gap-comparison} quantifies the difference.
Beyond identifiability, spherical representations also enjoy favorable worst-case guarantees for \(k\)-NN, linear probing, and kernel ridge regression~\citep{nicollier2026spherejepa}.

More generally, our results do not identify a universally optimal target.
When the latent world is known and linearly admissible, matching the representation target to its latent distribution makes the latent coordinates identifiable up to the symmetries of the world.
For unknown worlds, choosing a Gaussian, spherical, or other target amounts to selecting an inductive bias: Gaussian targets suit Euclidean Gaussian dynamics, but may distort representations of compact latent variables such as directions, phases, and rotations.

\section{Experiments}
\label{sec:experiments}

%\gf{We address three questions: whether matching the target to the latent geometry improves linear recovery, whether simulator-generated positive pairs can recover hidden physical variables from trajectories, and whether target mismatch persists in high dimension.}{
We address two main questions: whether matching the target to the latent geometry improves linear recovery, and whether target mismatch persists in high dimension.%}

\paragraph{Protocol.}
All experiments combine alignment with target-distribution regularization.
The geometry-comparison experiments use the normalized heat-kernel MMD of Appendix~\ref{app:spectral-mmd}; %\gf{the two-body experiment uses Gaussian SIGReg;}{} 
and the high-dimensional experiment uses SIGReg for Gaussian targets and product heat-kernel MMD for the Clifford target.

For each world--target pair, we select only the regularization weight \(\alpha\) using seed \(0\): we freeze each candidate encoder, define \(Y=f(X)\), fit an affine ridge probe \(Y\!\to\!Z\), and maximize validation \(R^2\).
For compact targets, heat times are fixed a priori such that the normalized kernel value between maximally distant points is \(0.1\)–\(0.2\); they are not selected by validation.
We then retrain the selected configuration with seeds \(1,2,3\) and report held-out test means and sample standard deviations.
This oracle selection uses latent labels only to choose \(\alpha\), giving mismatched targets a favorable best-case comparison.
Appendix~\ref{app:loss-recovery} further shows that the high-recovery regime can be diagnosed without latent probes: it occurs only when both the invariance loss and the distribution-matching loss approach their self-supervised reference values.

Our theory concerns exact distribution matching in the matched population setting.
Finite-weight matched experiments approximate this constraint, whereas mismatched experiments are controlled stress tests outside the formal scope of the theorems.

\subsection{Matched and Mismatched Target Geometries}
\label{sec:geometry-experiment}

We test all nine combinations of three latent worlds and representation targets of intrinsic dimension two---a Gaussian Cartesian position, a spherical viewing direction, and a two-angle toroidal configuration.% (Figure~\ref{fig:dataset_samples}).
Their canonical embeddings have different ambient dimensions; we match intrinsic dimension because the continuous image of a two-dimensional latent space cannot have a full-dimensional Gaussian distribution in a higher-dimensional ambient space.
All targets are enforced using the geometry-specific heat-kernel MMDs of Appendix~\ref{app:spectral-mmd}.

%\begin{figure}
%    \centering
%    % Robot cartésien (Gaussian Cartesian position)
%    \begin{minipage}{0.2\textwidth}
%        \centering
%        \includegraphics[width=\linewidth]{cartesian_robot.pdf}
%        \caption*{Cartesian Robot}
%    \end{minipage}\hfill
%    % Objet 3D (Spherical viewing direction)
%    \begin{minipage}{0.2\textwidth}
%        \centering
%        \includegraphics[width=\linewidth]{object3d.pdf}
%        \caption*{3D Object}
%    \end{minipage}\hfill
%   % Bras articulé (Two-angle toroidal configuration)
%    \begin{minipage}{0.2\textwidth}
%        \centering
%        \includegraphics[width=\linewidth]{arm.pdf}
%        \caption*{Articulated Arm}
%    \end{minipage}
%    \caption{Samples from the three latent worlds evaluated: Gaussian Cartesian position (left), spherical viewing direction (center), and two-angle toroidal configuration (right).}
%    \label{fig:dataset_samples}
%\end{figure}

\begin{table}[t]
    \centering
    \small
    \setlength{\tabcolsep}{1pt}
    \caption{
    Held-out test \(R^2(Y\!\to\!Z)\), measuring linear recovery of the latent state from the learned representation.
    Except where indicated, values are means and sample standard deviations over three seeds.}
    \label{tab:geometry-results-y-to-z}
    \begin{tabular}{lccc}
        \toprule
        Latent world
        & Gaussian
        & Sphere
        & Torus \\
        \midrule
        Gaussian
        & \(\mathbf{0.995 \pm 0.001}\)
        & \(0.909 \pm 0.009\)
        & \(0.606 \pm 0.175\) \\
        Sphere
        & \(0.669 \pm 0.005\)
        & \(\mathbf{0.998 \pm 0.001}\)
        & \(0.921 \pm 0.003\) \\
        Torus
        & \(0.493 \pm 0.011\)
        & \(0.737 \pm 0.004\)
        & \(\mathbf{0.996\pm0.001}^{\dagger}\) \\
        \bottomrule
    \end{tabular}
    \vspace{2pt}
\parbox{\columnwidth}{\scriptsize
    Columns indicate the representation target.
    \(\dagger\) Over ten seeds, five runs achieve near-perfect linear recovery; the reported value is the mean and sample standard deviation over these five runs. 
    The remaining runs are discussed in the text.}
\end{table}

Table~\ref{tab:geometry-results-y-to-z} shows that matching the representation target to the latent geometry yields the best linear recovery for each latent world.
The matched Gaussian and spherical targets recover their latent states almost perfectly, with \(R^2(Y\!\to\!Z)=0.995\pm0.001\) and \(0.998\pm0.001\), respectively.
For the toroidal world, optimization is sensitive to initialization.
Five out of ten runs achieve near-perfect linear recovery, with \(R^2(Y\!\to\!Z)=0.996\pm0.001\) and \(L_{\mathrm{inv}}=0.052\pm0.001\).
The other five runs obtain substantially lower recovery, \(R^2(Y\!\to\!Z)=0.597\pm0.187\), together with a higher invariance loss, \(L_{\mathrm{inv}}=0.093\pm0.017\).
This simultaneous increase in invariance loss and decrease in linear recovery is consistent with unsuccessful runs becoming trapped in poor local minima.
Intuitively, such solutions may encode only one of the two angular coordinates; escaping them requires learning the missing periodic coordinate while preserving approximate target matching.
The lower scores therefore reflect an optimization failure rather than a limitation of the matched toroidal geometry.

The two probe directions reveal an important asymmetry under target mismatch.
For example, in the toroidal world, the Gaussian representation is almost linearly predictable from the latent state, with \(R^2(Z\!\to\!Y)=0.986\), yet it allows only limited linear recovery of that state, with \(R^2(Y\!\to\!Z)=0.493\).
Thus, a mismatched target can produce a representation that is nearly linear in \(Z\) without being linearly invertible.
Appendix~\ref{app:geometry-experiment} reports the complete results in both probe directions.

\subsection{Does Geometry Matching Remain Useful in High Dimension?}
\label{sec:morse-scaling}

We test whether the benefit of geometry matching persists as the latent dimension increases.
For each even intrinsic dimension \(d\in\{4,8,16,32,64,128\}\), we consider the balanced Clifford torus \(\mathcal C_d=\mathbb S^{d/2}\times\mathbb S^{d/2}\subset\mathbb R^{d+2}\).

The latent state is the unobserved initial condition of two particles interacting through a Morse potential, and the encoder observes only fifty delayed states of their trajectory.
Positive pairs are generated by the stationary diffusion in \eqref{eq:sde} on \(\mathcal C_d\), with the transition time chosen to yield linear-mode correlation \(\rho=0.95\).

We compare the matched target with \(\mathcal N(0,I_d)\), which matches the intrinsic degrees of freedom, and with the favorable baseline \(\mathcal N(0,I_{d+2})\), whose extra coordinates can retain the canonical embedding of \(\mathcal C_d\subset\mathbb R^{d+2}\) under finite-weight Gaussian regularization.

The matched Clifford target achieves the highest recovery at every dimension.
At \(d=128\), it reaches \(R^2(Y\!\to\!Z)=0.9667\), compared with \(0.7948\) for \(\mathcal N(0,I_{d+2})\) and \(0.7465\) for \(\mathcal N(0,I_d)\).
Thus, removing the Gaussian rank bottleneck improves recovery but does not close the gap with the matched target.

The main comparison uses SIGReg for Gaussian targets and product heat-kernel MMD for the Clifford target; replacing SIGReg with a Matérn-MMD Gaussian regularizer preserves the ordering.
See Appendix~\ref{app:morse-details}.

\section{Conclusion}
\label{sec:conclusion}

We extended the linear-identifiability theory of World--Learner from Euclidean additive-noise models to latent distributions supported on embedded Riemannian manifolds.
Our results characterize when distribution-preserving alignment recovers the latent state up to an orthogonal symmetry and show that approximate recovery is controlled by the separation between linear and nonlinear spectral modes.

Across Gaussian, spherical, toroidal, and Clifford-torus worlds, matched target geometries provide the best linear recovery when optimization succeeds.
For the toroidal world, the matched target reaches near-perfect recovery in two of three runs, with one lower-recovery seed indicating residual optimization sensitivity.
A Gaussian target in $\mathbb R^d$ is therefore not geometry-neutral: even in high dimension, concentration does not eliminate its global geometric and topological mismatch with a manifold-supported latent world.
%\gf{In a simulator-access setting, the framework also recovers hidden initial conditions from nonlinear trajectory observations.}{}

Our characterization relies on the quadratic structure of the alignment loss.
A natural extension is to study $r$-homogeneous invariance objectives and their associated weighted $r$-Laplacians, which formally lead to generalized-Gaussian targets with heavier-than-Gaussian tails.
Establishing the corresponding link between nonlinear spectral structure and linear recovery remains an interesting direction for future work.

\subsection*{AI use statement}

In this work, we used generative AI tools extensively for software development. 
The authors designed the proposed method, specified the intended behavior and requirements of the code, and directed the development process.
Based on these instructions, generative AI tools generated a substantial portion of the surrounding research code, including data preparation and preprocessing scripts, experiment configuration and execution code, evaluation routines, and result-processing and visualization utilities.
The tools were also used to modify and debug the code and to assist with executing experimental runs.

Generative AI tools were additionally used throughout the preparation of the manuscript to draft, rephrase, and edit text and to improve its readability, language, organization, and presentation.
The research objectives, proposed method, experimental decisions, scientific claims, and interpretation of the final results were determined by the authors.

The authors reviewed the AI-assisted code and verified it by executing the relevant data-processing, training, and evaluation pipelines and by checking their outputs for consistency and correctness.
All AI-assisted manuscript content was reviewed and revised, and the reported results, technical statements, and references were checked against the experiments and original sources.
The authors take full responsibility for the final content of this work, including its text, claims, code, data processing, experimental results, and other artifacts produced with the assistance of generative AI.

\subsection*{Ethics statement}

This work focuses on methodological research and does not involve human participants, the collection of personal or sensitive information, or the release of a new dataset.
The experiments use existing datasets in accordance with their applicable licenses and intended research purposes. 
While the proposed method is not designed for a harmful application, its performance may depend on the composition and potential biases of the data on which it is trained and evaluated.
Consequently, results should be interpreted within the experimental settings considered in this work, and additional validation would be necessary before deployment in sensitive or high-stakes applications.
To the best of our knowledge, this work does not raise additional concerns regarding privacy, security, legal compliance, conflicts of interest, or research integrity.

\subsection*{Reproducibility statement}

The assumptions, definitions, and experimental protocol underlying our results are described in the main text. Complete proofs of the theoretical claims are provided in Appendix~\ref{sec:proofs-main-results}. Appendix~\ref{app:experimental-details} documents the latent-world and observation-generation procedures, model architectures, optimization settings, hyperparameter-selection protocol, random seeds, data splits, and evaluation methodology. The definitions, normalizations, and implementation details of the geometry-aware distribution-matching regularizers are given in Appendix~\ref{app:spectral-mmd}.
Additional experiments and seed-level results are reported in the remaining appendix sections. Upon acceptance, we will publicly release the complete source code for data generation, training, evaluation, and figure production, together with the experiment configurations and software dependencies needed to reproduce the reported results.

\subsubsection*{Acknowledgments}
This work was partially funded by ANRT with the CIFRE grant n° 2024/1221, Advanced Track and Trace, and Centre Borelli. 
This work was granted access to the HPC resources of IDRIS (Jean Zay supercomputer) under the allocation AD011017323 made by GENCI

\bibliography{ref}
\bibliographystyle{iclr2027_conference}

\newpage
\appendix

\section{Experimental Details and Supplementary Results}
\label{app:experimental-details}
\subsection{Matched and mismatched target geometries}
\label{app:geometry-experiment}

\paragraph{Latent worlds and observations.}
The Gaussian world uses $Z\sim\mathcal N(0, I_2)$ and an Ornstein--Uhlenbeck transition with first-mode correlation $\rho=0.95$.
Observations are $128\times128$ images generated by a procedural Cartesian-robot renderer, whose horizontal and vertical positions are $0.28Z_1$ and $0.28Z_2$.
The torus world samples two joint angles uniformly on $\mathbb T^2=\mathbb S^1\times\mathbb S^1$, evolves them through the stationary torus diffusion, and renders the resulting two-link arm using the \texttt{reacher-hard} environment from the DeepMind Control Suite and MuJoCo~\citep{tassa2018deepmind,todorov2012mujoco}. 
The spherical world samples a uniform camera direction on $\mathbb S^2$, evolves it through spherical Brownian motion, and renders the Stanford Bunny ~\citep{turk1994zippered} using PyVista ~\citep{sullivan2019pyvista}.
All observations have resolution $128\times128$.

\begin{figure}[h]
    \centering
    \includegraphics[width=.7\columnwidth]{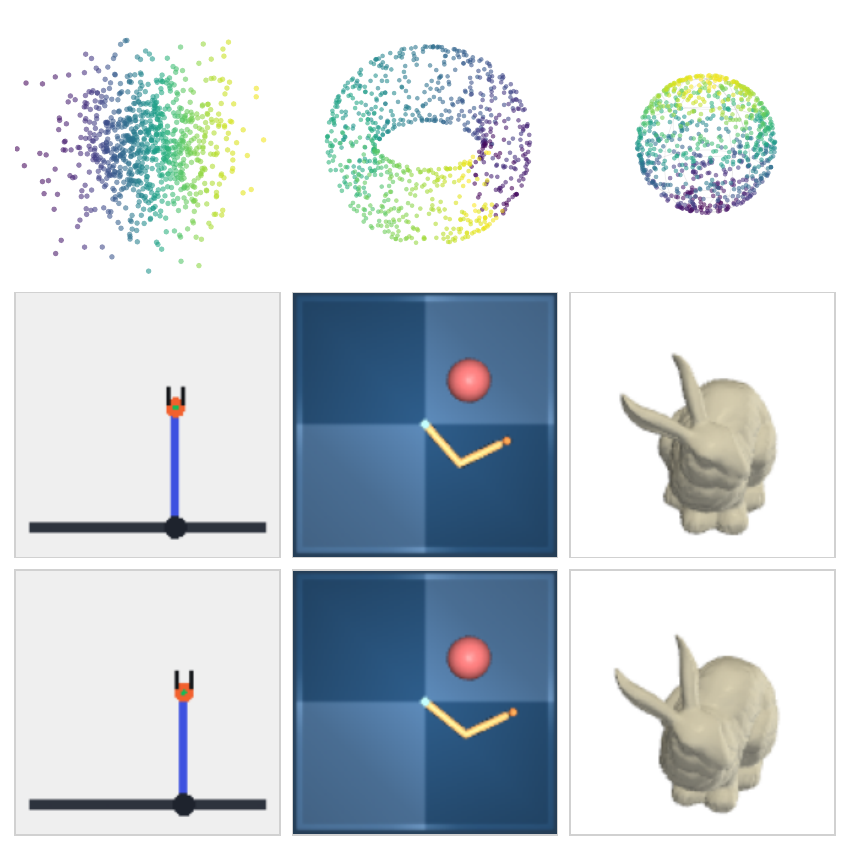}
    \caption{
        Latent worlds and corresponding observations.
        Columns show the Gaussian, toroidal, and spherical worlds. 
        The first row visualizes their latent distributions, while the remaining rows show positive observation pairs $(X,X')$.
    }
    \label{fig:latent-worlds}
\end{figure}

\paragraph{Targets and kernels.}
We combine each latent world with three representation targets: the standard Gaussian $\mathcal N(0,I_2)$, the uniform distribution on $\mathbb S^2$, and the uniform distribution on $\mathbb T^2$.
Each target is enforced using a geometry-specific heat-kernel MMD regularizer~\citep{gretton2012kernel}.
The complete kernel definitions, normalizations, and target expectations are provided in Appendix~\ref{app:spectral-mmd}.

For the Gaussian target, we use the Mehler kernel associated with the Ornstein--Uhlenbeck semigroup, with heat time $t=1$.
For the spherical target, we use the heat kernel on $\mathbb S^2$ with $t=0.125$.
For the toroidal target, we use the product of two heat kernels on $\mathbb S^1$, each with $t=1.5$.
These heat times are fixed a priori rather than selected by validation.
For the compact targets, they are chosen by a simple scale rule: after normalizing the kernel so that $k(x,x)=1$, the kernel value at an antipodal point is approximately $0.1$.
In every case, the kernel is defined with respect to the corresponding invariant target distribution, which makes the target-dependent MMD expectations analytically tractable.
The regularizers are normalized as described in Appendix~\ref{app:spectral-mmd}, and their heat times are held fixed across all latent worlds.

\paragraph{Optimization and model selection.}
The encoder is an ImageNet-1K-pretrained EdgeNeXt-XX-Small backbone~\citep{maaz2022edgenext} followed by a linear projection into the appropriate ambient space. 
We use AdamW for $50$ epochs, batches of $256$ positive pairs, a learning rate of $5\times10^{-2}$, weight decay $10^{-4}$, and a one-epoch linear warm-up. Each epoch contains $100$ freshly generated training minibatches.
For each world--target pair, we train seed $0$ over $\alpha\in\{10^{-3},5\times10^{-3},10^{-2},5\times10^{-2},10^{-1}\}$. 
For each candidate, the encoder is frozen, and a closed-form affine ridge probe is fitted on the training split.
We select $\alpha$ by validation $R^2(Y\!\to\!Z)$, then retrain the selected configuration with seeds $1$, $2$, and $3$.
For the matched Torus-to-Torus configuration, we additionally evaluate ten independent seeds.

\paragraph{Evaluation.}
For every final seed, we freeze the encoder and fit closed-form affine ridge probes $Y\!\to\!Z$ and $Z\!\to\!Y$ on the training split. 
Their raw affine predictions are evaluated on the held-out test split, without projection onto the target manifold.
The test split is never used for selecting \(\alpha\) or for model evaluation choices.
Except for the matched Torus-to-Torus configuration, we report the mean and sample standard deviation over seeds \(1\), \(2\), and \(3\).
For Torus-to-Torus, we analyze ten independently trained seeds.
\begin{table*}[t]
    \centering
    \caption{
    Held-out test $R^2$ between the learned representation $Y$ and the
    ground-truth state $Z$. Values are means and sample standard deviations over three independently trained seeds, except for the
    marked Torus target entries.
}
    \label{tab:geometry-results-full}
    \resizebox{\linewidth}{!}{
    \begin{tabular}{lcccccc}
        \toprule
        & \multicolumn{2}{c}{Gaussian target}
        & \multicolumn{2}{c}{Spherical target}
        & \multicolumn{2}{c}{Torus target} \\
        Latent world
        & $Y\!\to\!Z$ & $Z\!\to\!Y$
        & $Y\!\to\!Z$ & $Z\!\to\!Y$
        & $Y\!\to\!Z$ & $Z\!\to\!Y$ \\
        \midrule
        Gaussian
        & $0.995 \mathbin{\pm} 0.001$
        & $0.995 \mathbin{\pm} 0.001$
        & $0.909 \mathbin{\pm} 0.009$
        & $0.748 \mathbin{\pm} 0.015$
        & $0.606 \mathbin{\pm} 0.175$
        & $0.322 \mathbin{\pm} 0.079$ \\
        Sphere
        & $0.669 \mathbin{\pm} 0.005$
        & $0.997 \mathbin{\pm} 0.001$
        & $0.998 \mathbin{\pm} 0.001$
        & $0.998 \mathbin{\pm} 0.001$
        & $0.921 \mathbin{\pm} 0.003$
        & $0.742 \mathbin{\pm} 0.002$ \\
        Torus
        & $0.493 \mathbin{\pm} 0.011$
        & $0.986 \mathbin{\pm} 0.022$
        & $0.737 \mathbin{\pm} 0.004$
        & $0.983 \mathbin{\pm} 0.001$
        & $0.996 \mathbin{\pm} 0.001^{\dagger}$
        & $0.996 \mathbin{\pm} 0.001^{\dagger}$ \\ \\
        \midrule
        \multicolumn{7}{p{1.0\linewidth}}{
        \scriptsize \(\dagger\) Over ten seeds, five runs achieve near-perfect linear recovery; the reported values are the means and sample standard deviations over these five runs.
        The remaining runs are discussed in the text.
} \\
        \bottomrule
    \end{tabular}
    }
\end{table*}

\paragraph{Results.}
As shown in Table~\ref{tab:geometry-results-full}, matching the representation target to the latent geometry yields the best \(R^2(Y\!\to\!Z)\) for each latent world.
The Gaussian and spherical matched models recover their latent states almost perfectly, reaching respectively $0.995\pm0.001$ and $0.998\pm0.001$.

The two probe directions reveal why target mismatch matters.
For example, in the spherical world, the Gaussian representation is almost perfectly predictable from the latent state,
with $R^2(Z\!\to\!Y)=0.997$, but the latent state is not fully recoverable from that representation,
with $R^2(Y\!\to\!Z)=0.669$.
The same asymmetry is stronger in the toroidal world, where the Gaussian target gives
$R^2(Z\!\to\!Y)=0.986$ but only
$R^2(Y\!\to\!Z)=0.493$.
Thus, a mismatched representation can remain a predictable function of the latent state while losing information required to reconstruct it linearly.
The geometric distortions associated with this loss are visualized in Figure~\ref{fig:geometry-teaser}.

For the toroidal world, optimization is sensitive to initialization.
Five out of ten runs achieve near-perfect linear recovery, with \(R^2(Y\!\to\!Z)=0.996\pm0.001\) and \(L_{\mathrm{inv}}=0.052\pm0.001\).
The other five runs obtain substantially lower recovery, \(R^2(Y\!\to\!Z)=0.597\pm0.187\), together with a higher invariance loss, \(L_{\mathrm{inv}}=0.093\pm0.017\).
This simultaneous increase in invariance loss and decrease in linear recovery indicates that unsuccessful runs remain trapped in poor local minima.
The lower scores therefore reflect an optimization failure rather than a limitation of the matched toroidal geometry.

\subsection{High-Dimensional Morse Dynamics on Clifford Tori}
\label{app:morse-details}

\paragraph{World and observations.}
For each even \(d\in\{4,8,16,32,64,128\}\), let \(m=d/2\) and
\[
\mathcal C_d
=
\mathbb S^m\times\mathbb S^m
\subset\mathbb R^{d+2},
\qquad
Z=(u_0,v_0)\sim\operatorname{Unif}(\mathcal C_d).
\]
Positive pairs are generated independently on the two spherical factors using the Brownian transition of~\eqref{eq:sde}, with ten tangent-space steps and transition time chosen so that
\[
\mathbb E\langle u,u'\rangle
=
\mathbb E\langle v,v'\rangle
=
\rho=0.95.
\]

The observation map \(g\) evolves two particles on the spherical factors under the Morse potential
\[
V_d(u,v)
=
D_d
\left[
1-
\exp\!\left(
-a_d\bigl(r_\varepsilon(u,v)-r_\star\bigr)
\right)
\right]^2,
\qquad
r_\varepsilon(u,v)
=
\sqrt{\|u-v\|^2+\varepsilon^2},
\]
where
\[
a_d=\sqrt{m+1},
\qquad
D_d=a_d^{-1},
\qquad
r_\star=\sqrt2,
\qquad
\varepsilon=10^{-3}.
\]
The scaling \(D_da_d=1\) keeps the force prefactor constant across dimensions.
Starting from \((u_0,v_0)\) with zero velocity, we integrate the corresponding Riemannian dynamics with step size \(0.01\), using geodesic position updates and tangent velocity projections.
The encoder does not observe the initial condition nor the first \(200\) integration steps; its input is the next \(50\) states,
\[
X=g(Z)
=
\bigl((u(t_j),v(t_j))\bigr)_{j=0}^{49}
\in\mathbb R^{50\times(d+2)},
\qquad
t_j=0.01(200+j).
\]

\paragraph{Encoder and targets.}
All targets share an equivariant temporal encoder.
At each time, it forms rotation-invariant Gram features from
\(u,v,\dot u,\dot v\), processes them with a temporal convolutional network of width \(128\) and dilations
\((1,2,4,8,16,32)\), and uses the resulting scalar coefficients to combine the four input vectors equivariantly across time.

We compare three representation targets:
\[
\mathcal N(0,I_d),
\qquad
\mathcal N(0,I_{d+2}),
\qquad
\operatorname{Unif}(\mathcal C_d).
\]
The first Gaussian target matches the intrinsic dimension.
The second provides \(d+2\) output coordinates, allowing the representation to retain the canonical Clifford embedding despite its Gaussian regularization.
The matched target also uses \(d+2\) coordinates, split into two normalized \((m+1)\)-dimensional blocks.

The main Gaussian models use SIGReg, whereas the Clifford model uses a product heat-kernel MMD.
For \(d=(4,8,16,32,64,128)\), its respective kernel temperatures are
\[
t=(0.5,0.5,0.375,0.1875,0.125,0.078125).
\]

\paragraph{Training and evaluation.}
We train for \(50\) epochs with AdamW, batch size \(256\), learning rate \(3\times10^{-4}\), zero weight decay, and one epoch of linear warm-up.
Each epoch contains \(50\) newly generated batches; validation and test sets contain \(10\) and \(40\) fixed batches.
Using seed \(0\), we select
\(\lambda_{\mathrm{reg}}\in\{0.01,0.05,0.1\}\)
by validation \(R^2(Y\!\to\!Z)\), then retrain with seeds \(1,2,3\).
Affine ridge probes are fitted on the training split and evaluated on the held-out test split.
We report test means and sample standard deviations.

\begin{table*}[t]
\centering
\small
\setlength{\tabcolsep}{4.5pt}
\caption{
Linear recovery from delayed Morse trajectories.
Results are test means and sample standard deviations over seeds \(1,2,3\), after selecting \(\lambda_{\mathrm{reg}}\) using seed \(0\).
}
\label{tab:morse-scaling-results}
\begin{tabular}{rllcc}
\toprule
\(d\) & Representation target & Output dim. &
\(R^2(Y\!\to\!Z)\) & \(R^2(Z\!\to\!Y)\)\\
\midrule
4   & \(\mathcal N(0,I_d)\)       & \(d\)   & \(0.4540\pm0.0298\) & \(0.6985\pm0.0795\)\\
    & \(\mathcal N(0,I_{d+2})\)   & \(d+2\) & \(0.6755\pm0.0148\) & \(0.6849\pm0.0101\)\\
    & \(\operatorname{Unif}(\mathcal C_d)\) & \(d+2\) & \(\mathbf{0.8315\pm0.0119}\) & \(\mathbf{0.8362\pm0.0103}\)\\
\midrule
8   & \(\mathcal N(0,I_d)\)       & \(d\)   & \(0.5668\pm0.0194\) & \(0.7561\pm0.0686\)\\
    & \(\mathcal N(0,I_{d+2})\)   & \(d+2\) & \(0.7271\pm0.0010\) & \(0.7473\pm0.0053\)\\
    & \(\operatorname{Unif}(\mathcal C_d)\) & \(d+2\) & \(\mathbf{0.8757\pm0.0038}\) & \(\mathbf{0.8796\pm0.0036}\)\\
\midrule
16  & \(\mathcal N(0,I_d)\)       & \(d\)   & \(0.6590\pm0.0134\) & \(0.8599\pm0.0416\)\\
    & \(\mathcal N(0,I_{d+2})\)   & \(d+2\) & \(0.7748\pm0.0035\) & \(0.8242\pm0.0121\)\\
    & \(\operatorname{Unif}(\mathcal C_d)\) & \(d+2\) & \(\mathbf{0.9266\pm0.0096}\) & \(\mathbf{0.9288\pm0.0083}\)\\
\midrule
32  & \(\mathcal N(0,I_d)\)       & \(d\)   & \(0.7168\pm0.0126\) & \(0.9346\pm0.0126\)\\
    & \(\mathcal N(0,I_{d+2})\)   & \(d+2\) & \(0.8148\pm0.0054\) & \(0.9101\pm0.0233\)\\
    & \(\operatorname{Unif}(\mathcal C_d)\) & \(d+2\) & \(\mathbf{0.9518\pm0.0098}\) & \(\mathbf{0.9536\pm0.0085}\)\\
\midrule
64  & \(\mathcal N(0,I_d)\)       & \(d\)   & \(0.7000\pm0.0690\) & \(0.9512\pm0.0060\)\\
    & \(\mathcal N(0,I_{d+2})\)   & \(d+2\) & \(0.8224\pm0.0124\) & \(0.9508\pm0.0111\)\\
    & \(\operatorname{Unif}(\mathcal C_d)\) & \(d+2\) & \(\mathbf{0.9542\pm0.0038}\) & \(\mathbf{0.9616\pm0.0022}\)\\
\midrule
128 & \(\mathcal N(0,I_d)\)       & \(d\)   & \(0.7465\pm0.0651\) & \(0.9599\pm0.0030\)\\
    & \(\mathcal N(0,I_{d+2})\)   & \(d+2\) & \(0.7948\pm0.0481\) & \(0.9408\pm0.0208\)\\
    & \(\operatorname{Unif}(\mathcal C_d)\) & \(d+2\) & \(\mathbf{0.9667\pm0.0024}\) & \(\mathbf{0.9680\pm0.0018}\)\\
\bottomrule
\end{tabular}
\end{table*}

\paragraph{Gaussian Matérn-MMD control.}
To verify that the comparison is not specific to SIGReg, we repeat both Gaussian baselines using a Matérn-\(3/2\) MMD.
Table~\ref{tab:morse-gaussian-matern} reports the selected regularization weights and test results.
The matched Clifford results in Table~\ref{tab:morse-scaling-results} remain higher than both Matérn-MMD Gaussian controls at every shared dimension.

\begin{table*}[t]
\centering
\small
\setlength{\tabcolsep}{5pt}
\caption{
Gaussian Matérn-\(3/2\) MMD controls.
Results are test means and sample standard deviations over seeds \(1,2,3\); \(\lambda_{\mathrm{reg}}\) is selected using seed \(0\).
}
\label{tab:morse-gaussian-matern}
\begin{tabular}{rllccc}
\toprule
\(d\) & Gaussian target & Output dim. & Selected \(\lambda_{\mathrm{reg}}\) &
\(R^2(Y\!\to\!Z)\) & \(R^2(Z\!\to\!Y)\)\\
\midrule
4   & \(\mathcal N(0,I_d)\)       & \(d\)   & \(0.01\) & \(0.5243\pm0.0073\) & \(0.9093\pm0.0407\)\\
    & \(\mathcal N(0,I_{d+2})\)   & \(d+2\) & \(0.01\) & \(0.7533\pm0.0068\) & \(0.8016\pm0.0320\)\\
8   & \(\mathcal N(0,I_d)\)       & \(d\)   & \(0.01\) & \(0.6342\pm0.0250\) & \(0.9297\pm0.0414\)\\
    & \(\mathcal N(0,I_{d+2})\)   & \(d+2\) & \(0.01\) & \(0.7843\pm0.0041\) & \(0.8545\pm0.0366\)\\
16  & \(\mathcal N(0,I_d)\)       & \(d\)   & \(0.05\) & \(0.6493\pm0.0139\) & \(0.8515\pm0.0389\)\\
    & \(\mathcal N(0,I_{d+2})\)   & \(d+2\) & \(0.05\) & \(0.7637\pm0.0043\) & \(0.8165\pm0.0156\)\\
32  & \(\mathcal N(0,I_d)\)       & \(d\)   & \(0.01\) & \(0.7648\pm0.0138\) & \(0.9650\pm0.0040\)\\
    & \(\mathcal N(0,I_{d+2})\)   & \(d+2\) & \(0.01\) & \(0.8492\pm0.0114\) & \(0.9840\pm0.0125\)\\
64  & \(\mathcal N(0,I_d)\)       & \(d\)   & \(0.01\) & \(0.7861\pm0.0131\) & \(0.9626\pm0.0052\)\\
    & \(\mathcal N(0,I_{d+2})\)   & \(d+2\) & \(0.05\) & \(0.8102\pm0.0129\) & \(0.9397\pm0.0128\)\\
128 & \(\mathcal N(0,I_d)\)       & \(d\)   & \(0.01\) & \(0.8139\pm0.0160\) & \(0.9666\pm0.0026\)\\
    & \(\mathcal N(0,I_{d+2})\)   & \(d+2\) & \(0.01\) & \(0.8463\pm0.0113\) & \(0.9601\pm0.0120\)\\
256 & \(\mathcal N(0,I_d)\)       & \(d\)   & \(0.01\) & \(0.8289\pm0.0199\) & \(0.9691\pm0.0033\)\\
    & \(\mathcal N(0,I_{d+2})\)   & \(d+2\) & \(0.01\) & \(0.8553\pm0.0078\) & \(0.9543\pm0.0143\)\\
\bottomrule
\end{tabular}
\end{table*}

\paragraph{Results.}
The \(d\)-dimensional Gaussian target has the natural number of stochastic degrees of freedom but fewer output coordinates than the canonical Clifford state in \(\mathbb R^{d+2}\).
Increasing its output dimension to \(d+2\) generally improves recovery, showing that output dimensionality accounts for part of the gap.
However, this favorable Gaussian control remains geometrically mismatched: under finite-weight regularization, it must trade off retaining the manifold-supported latent state against approaching a full-dimensional Gaussian distribution.

The matched Clifford target avoids this tradeoff and achieves the highest \(R^2(Y\!\to\!Z)\) at every shared dimension under both Gaussian regularizers.
At \(d=128\), it reaches \(0.9667\), compared with \(0.7948\) for the \(d+2\)-dimensional SIGReg Gaussian and \(0.8463\) for its Matérn-MMD counterpart.
Thus, additional Gaussian coordinates and a different distribution-matching regularizer both improve recovery, but neither closes the gap with a target matching the compact latent geometry.

\subsection{Relationship Between Training Losses and Latent Recovery}
\label{app:loss-recovery}

Our main experiments select the regularization weight using validation latent recovery.
Here, we examine whether the self-supervised training losses identify the same high-recovery regime.
On the matched spherical world, we sweep \(\alpha\in\{10^{-3},5\times10^{-3},10^{-2},5\times10^{-2},10^{-1},5\times10^{-1}\}\) over three independent seeds.
For each trained encoder, we report the validation invariance loss \(L_{\mathrm{inv}}\), distribution-matching loss \(L_{\mathrm{reg}}\), and linear-probe recovery \(R^2(Y\!\to\!Z)\).

The population objective \(\mathcal L_p\) used in our theory is defined using the full squared Euclidean distance.
The implementation reports its rescaled version
\[
L_{\mathrm{inv}}
=
\frac{1}{2}\mathbb E\!\left[\|f(X)-f(X')\|^2\right]
=
1-\mathbb E\!\left[\langle f(X),f(X')\rangle\right],
\]
Consequently, for \(\rho=0.95\), the constrained optimum under this
experimental convention is \(L_{\mathrm{inv}}^\star=1-\rho=0.05\).
The normalized finite-batch MMD satisfies \(\mathbb E[L_{\mathrm{reg}}]=1\) for samples drawn from the spherical target.

\begin{figure}[t]
    \centering
    \includegraphics[width=0.5\linewidth]{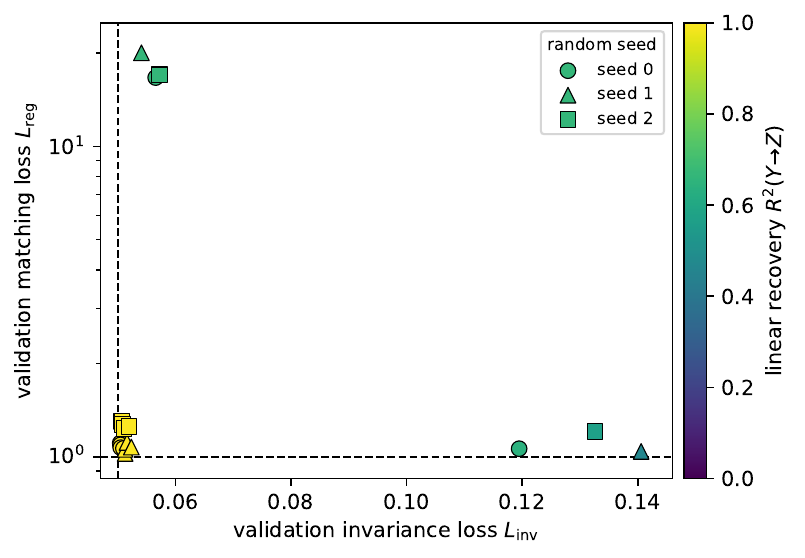}
    \caption{
        Self-supervised validation losses and latent recovery on the matched spherical world.
        Each point represents one regularization weight and seed.
        Marker shape indicates the seed and color indicates \(R^2(Y\!\to\!Z)\).
        The dashed lines denote \(L_{\mathrm{inv}}^\star=0.05\) and the target-sampling reference \(\mathbb E[L_{\mathrm{reg}}]=1\).
    }
    \label{fig:loss-recovery}
\end{figure}

Figure~\ref{fig:loss-recovery} shows that high recovery occurs only when both losses approach their respective references.
For \(\alpha\in[5\times10^{-3},10^{-1}]\), this regime consistently yields
\(R^2(Y\!\to\!Z)>0.99\) across all seeds.
By contrast, \(\alpha=10^{-3}\) produces poor distribution matching despite a near-optimal invariance loss, whereas \(\alpha=0.5\) matches the target distribution but has a substantially larger invariance loss.
Thus, neither loss alone is sufficient, while their joint behavior provides an unsupervised criterion for selecting the high-recovery regime.

\subsection{Empirical Behavior of the Approximate-Recovery Bound}
\label{app:approximate-bound-experiment}

We test whether learned Clifford-torus representations behave consistently with Theorem~\ref{thm:approximate-identifiability}.
Because distribution matching is enforced using finite-weight, finite-batch MMD, this experiment is a practical diagnostic rather than a verification of the population theorem.

\paragraph{Setup.}
For each even intrinsic dimension \(d\in\{4,8,16,32,64,128\}\), with \(m=d/2\), we sample
\[
Z=(U,V)\sim\operatorname{Unif}(\mathcal C_d),
\qquad
\mathcal C_d=\mathbb S^m\times\mathbb S^m
\subset\mathbb R^{d+2}.
\]
Positive pairs follow the stationary diffusion in~\eqref{eq:sde}, with transition time chosen so that the linear-mode correlation is \(\rho=0.95\).

Observations \(X=g(Z)\) are generated by a fixed invertible coupling that alternately transforms the two spherical factors.
The coupling applies three conditional orthogonal maps, each composed of two Householder reflections and parameterized by a separate frozen three-linear-layer network of hidden width \(64\).
Reversing the reflections explicitly inverts the coupling, so \(g\) satisfies Assumption~\ref{ass:world}.
Figure~\ref{fig:clifford-bound-experiment} (left) illustrates the resulting deformation.

The encoder \(f\) has three hidden layers of width \(512\) with SiLU activations and a \(d+2\)-dimensional output.
Its two \((m+1)\)-dimensional output blocks are independently normalized, yielding
\[
Y=h(Z)=\operatorname{Proj}_{\mathcal C_d}(f(X))
\in\mathcal C_d.
\]
We train for \(50\) epochs with AdamW, batch size \(256\), learning rate \(3\times10^{-2}\), regularization weight \(\lambda_{\mathrm{reg}}=0.005\), and seeds \(0,1,2\).
For \(d=4,8,16,32,64,128\), the respective MMD temperatures are
\(0.5,0.5,0.375,0.1875,0.125,\) and \(0.078125\).

\paragraph{Bound evaluation.}
Let \(D=d+2\) denote the ambient dimension.
Since \(\operatorname{Cov}(Z)=c_pI_D\) with \(c_p=2/D\), the optimal alignment loss is
\[
\mathcal L^\star
=2Dc_p(1-\rho)
=4(1-\rho)
=0.2.
\]
For each model, we estimate
\[
\widehat\varepsilon
=
\widehat{\mathcal L}_{\mathrm{inv}}-\mathcal L^\star,
\qquad
\widehat A
=
\arg\min_A\frac{1}{N}\sum_{i=1}^{N}
\|Y_i-AZ_i\|^2,
\qquad
\widehat E_{\mathrm{lin}}
=
\frac{1}{N}\sum_{i=1}^{N}
\|Y_i-\widehat AZ_i\|^2.
\]
For the balanced Clifford torus,
\(\lambda_{\mathrm{nl}}=2\lambda\), hence
\(\rho_{\mathrm{nl}}=\rho^2\).
Theorem~\ref{thm:approximate-identifiability} therefore predicts
\[
\widehat E_{\mathrm{lin}}
\leq
\widehat B
:=
\frac{\widehat\varepsilon}{2(\rho-\rho^2)}
=
\frac{\widehat\varepsilon}{0.095}.
\]

\begin{figure*}[t]
    \centering
    \begin{minipage}[c]{0.55\textwidth}
        \centering
        \includegraphics[width=\linewidth]{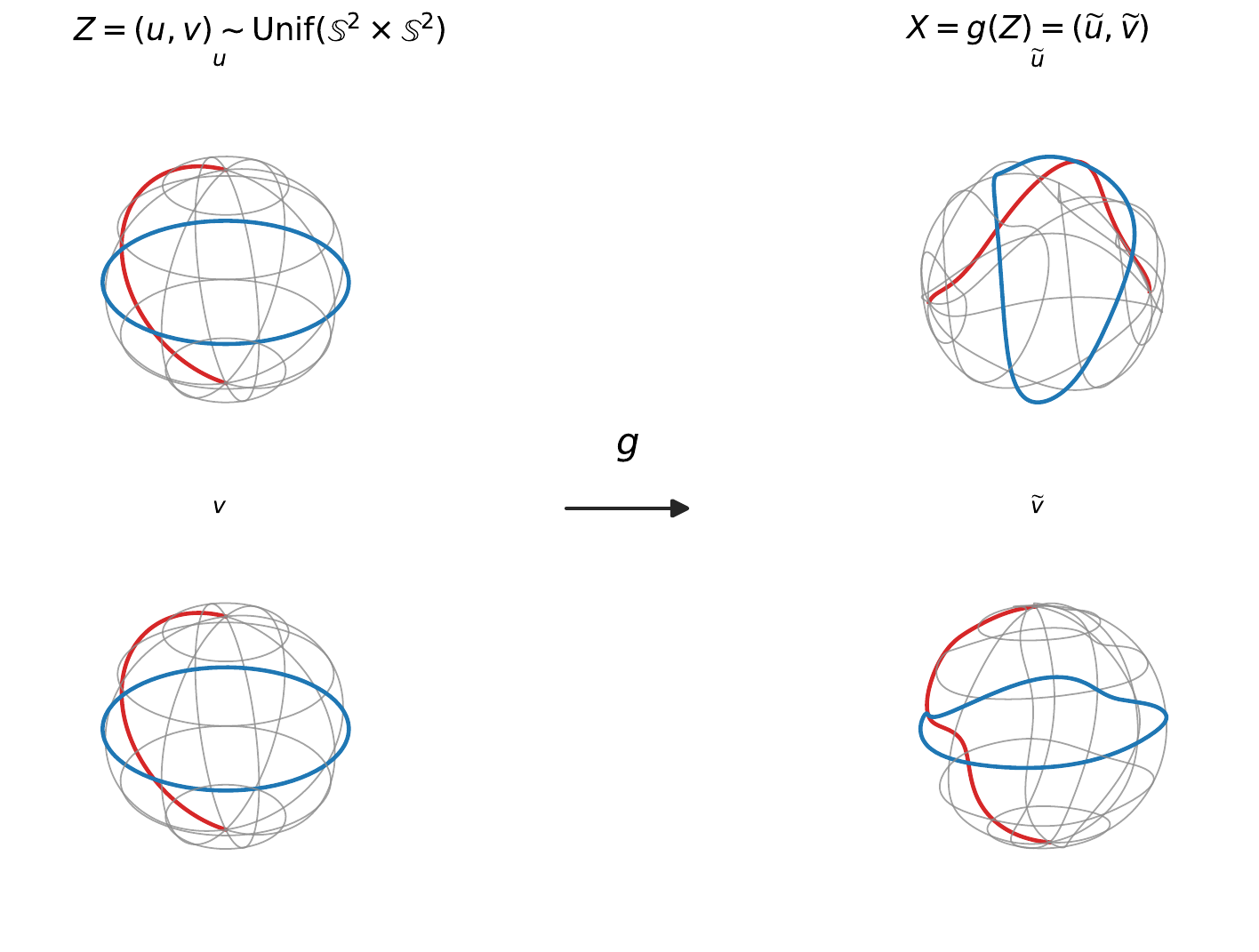}
    \end{minipage}
    \hfill
    \begin{minipage}[c]{0.41\textwidth}
        \centering
        \includegraphics[width=\linewidth]{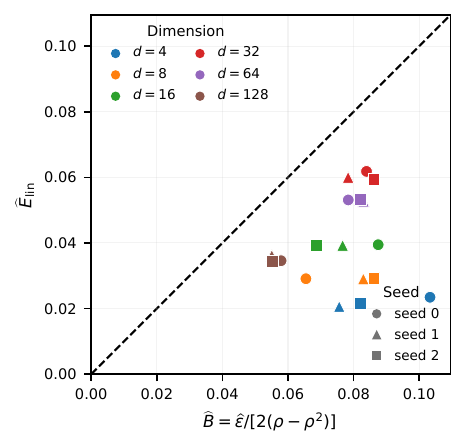}
    \end{minipage}
    \caption{
        \textbf{Left:} Observation map for
        \(\mathcal C_4=\mathbb S^2\times\mathbb S^2\), showing coordinate grids on the latent spherical factors and their images under the nonlinear invertible coupling \(g\).
        \textbf{Right:} Empirical behavior of the approximate-recovery bound across six dimensions and three seeds.
        The axes show the predicted upper bound \(\widehat B\) and the measured deviation \(\widehat E_{\mathrm{lin}}\) from the best linear map.
        The dashed diagonal denotes equality.
    }
    \label{fig:clifford-bound-experiment}
\end{figure*}

All \(18\) dimension--seed runs lie below the diagonal in Figure~\ref{fig:clifford-bound-experiment} (right).
Across runs, \(\widehat E_{\mathrm{lin}}\) ranges from \(0.0206\) to \(0.0618\), while \(\widehat B\) ranges from \(0.0551\) to \(0.1033\).
The MMD remains sufficiently small across all runs to indicate that the learned representations closely approximate the target distribution.
Thus, the learned representations behave consistently with the predicted inequality under approximate finite-sample distribution matching.

\section{Geometry-aware MMD with heat kernels}
\label{app:spectral-mmd}

Let $q$ denote the distribution of the learned representations and $p$ the target distribution.
For a positive-definite kernel $k$, define
\[
m_p(y)=\int_{\mathcal M}k(y,z)\,dp(z),
\qquad
c_p=\iint_{\mathcal M\times\mathcal M}k(z,z')\,dp(z)\,dp(z').
\]
The squared maximum mean discrepancy between $q$ and $p$ is
\begin{equation}
\operatorname{MMD}_k^2(q,p)
=
\mathbb E_{Y,Y'\sim q}[k(Y,Y')]
-2\mathbb E_{Y\sim q}[m_p(Y)]
+c_p.
\label{eq:mmd-primal}
\end{equation}
Given a batch of representations $y_1,\ldots,y_n\sim q$, we use the integrated estimator
\begin{equation}
\widehat{\operatorname{MMD}}_k^2
=
\frac{1}{n^2}\sum_{i,j=1}^n k(y_i,y_j)
-\frac{2}{n}\sum_{i=1}^n m_p(y_i)
+c_p.
\label{eq:integrated-mmd}
\end{equation}
Only the first term is estimated from the learned representations.
The expectations involving the target are computed once, either analytically or by deterministic quadrature.
This avoids sampling from the target during training.

\paragraph{Heat kernels adapted to the target geometry.}
The choice of kernel determines which differences between $q$ and $p$ are detected by MMD~\citep{gretton2012kernel}.
We use the heat kernel associated with the weighted Laplacian $D_p$ introduced in Section~\ref{sec:world}.
Let
\[
-D_p\phi_\ell=\lambda_\ell\phi_\ell,
\qquad
0=\lambda_0<\lambda_1\leq\lambda_2\leq\cdots,
\]
where $(\phi_\ell)_\ell$ is an orthonormal basis of $L^2(p)$ and $\phi_0=1$.
The corresponding heat kernel is
\begin{equation}
\widetilde k_t(y,y')
=
\sum_{\ell\geq0}
e^{-t\lambda_\ell}\phi_\ell(y)\phi_\ell(y'),
\qquad t>0.
\label{eq:heat-kernel-spectral}
\end{equation}
It is therefore adapted both to the geometry of $\mathcal M$ and to the target distribution $p$~\citep{hsu2002stochastic,bakry2014analysis}.
The resulting discrepancy can be written as
\begin{equation}
\operatorname{MMD}_{\widetilde k_t}^2(q,p)
=
\sum_{\ell\geq1}
e^{-t\lambda_\ell}
\left(\mathbb E_q[\phi_\ell]\right)^2.
\label{eq:heat-mmd-spectral}
\end{equation}
Thus, the regularizer compares $q$ and $p$ along every nonconstant spectral mode.
The heat time $t$ controls the scale of the comparison: smaller values retain finer geometric variations, whereas larger values emphasize smoother, large-scale discrepancies.

\paragraph{Product manifolds.}
Heat kernels are particularly convenient when the target space is a product manifold.
If
\[
\mathcal M=\mathcal M_1\times\cdots\times\mathcal M_r,
\qquad
p=p_1\otimes\cdots\otimes p_r,
\]
then
\begin{equation}
\widetilde k_t^{\mathcal M}(y,y')
=
\prod_{a=1}^r
\widetilde k_t^{\mathcal M_a}(y_a,y_a').
\label{eq:product-heat-kernel}
\end{equation}
A kernel on a torus can therefore be constructed by multiplying circle kernels, while a kernel on a Clifford torus can be constructed by multiplying the kernels of its two spherical factors.
This avoids computing the spectrum of the full product manifold.

\paragraph{Regularizer used in practice.}
To compare regularization strengths across target geometries, we divide the MMD by the discrepancy produced by a fully collapsed representation:
\begin{equation}
\widehat D_{\mathrm{MMD}}(q,p)
=
\frac{
\widehat{\operatorname{MMD}}_k^2
}{
\operatorname{MMD}_k^2(\delta_{z_\star},p)
},
\label{eq:normalized-heat-mmd}
\end{equation}
where $z_\star=0$ for a Gaussian target and $z_\star$ is any point on $\mathcal M$ for a uniform homogeneous-manifold target.
By symmetry, the latter value does not depend on the chosen point.
This normalization assigns unit penalty to a constant representation.
The denominator and the expectations involving $p$ are computed once before training, analytically or by deterministic quadrature; during training, only kernel evaluations between representations in the current batch are required.

\paragraph{Kernels used in the experiments.}
For the Gaussian target $p=\mathcal N(0,I_d)$, we use the Mehler kernel associated with the Ornstein--Uhlenbeck semigroup~\citep{bakry2014analysis}.
For uniform spherical targets, we evaluate the exact spherical heat kernel through its Gegenbauer expansion~\citep{zhao2018exact,nicollier2026expandingspherejepa}.
Torus kernels are products of the Fourier heat kernels of their circular factors, and Clifford-torus kernels are products of the heat kernels of their two spherical factors~\citep{hsu2002stochastic}.

In the matched- and mismatched-geometry experiments, we use $t=1$ for Gaussian and toroidal targets and $t=0.125$ for spherical targets.
The dimension-dependent heat times used in the Clifford-torus experiment are reported in Appendix~\ref{app:morse-details}.

\paragraph{Sobolev kernels as an alternative.}
The heat kernel weights the spectral mode $\phi_\ell$ by
$e^{-t\lambda_\ell}$, which rapidly suppresses high-frequency discrepancies.
A more sensitive alternative is to use the slower polynomial weighting
\[
k_s^{\mathrm{Sob}}(x,y)
=
\sum_{\ell\geq0}
(1+\lambda_\ell)^{-s}
\phi_\ell(x)\phi_\ell(y),
\]
for a sufficiently large $s$.
For $\mathcal M=\mathbb R^d$, Matérn kernels provide a standard example of
Sobolev kernels and may be used to match representations to a Gaussian target.
They correspond to the Euclidean Laplacian, rather than to the
Ornstein--Uhlenbeck generator associated specifically with the Gaussian measure.

\section{Canonical Worlds and Linear Admissibility}
\label{app:canonical-worlds}

A latent world \(\mathcal W=(\mathcal M,p,g,K_t)\) specifies the latent space \(\mathcal M\), its stationary distribution \(p\), the observation map \(g\), and the positive-pair transition \(K_t\).
In the examples below, \(K_t\) is generated by the stationary diffusion in~\eqref{eq:sde}, while \(g\) may be any measurable observation map that is injective \(p\)-almost surely.
Linear admissibility additionally depends on the embedding of \(\mathcal M\): its ambient coordinates must satisfy the density and geometric compatibility conditions of Definition~\ref{def:linearly-admissible} and form the complete first nonconstant eigenspace.
The same intrinsic manifold may therefore be admissible under one embedding but not another.

\paragraph{Gaussian world.}
Let \(\mathcal M=\mathbb R^d\) and \(p=\mathcal N(0,I_d)\).
The diffusion is Ornstein--Uhlenbeck, with time-\(t\) transition
\[
Z'=\rho Z+\sqrt{1-\rho^2}\,\varepsilon,
\qquad
\varepsilon\sim\mathcal N(0,I_d),
\qquad
\rho=e^{-t}.
\]
The Gaussian density satisfies the density condition at scale \(\lambda=1\), while the Euclidean embedding satisfies the geometric condition trivially.
The coordinates \(z_1,\ldots,z_d\) form the complete first nonconstant eigenspace, so the Gaussian world is linearly admissible.
The first nonlinear eigenvalue additionally satisfies
\(\lambda_{\mathrm{nl}}/\lambda=2\), which determines its approximate-recovery separation.

\paragraph{Spherical world.}
Let
\(\mathcal M=\mathbb S^m(r)\subset\mathbb R^{m+1}\) and
\(p=\operatorname{Unif}(\mathbb S^m(r))\).
Because the radius is constant, the density condition reduces to uniformity.
Moreover,
\[
\Delta_{\mathcal M}z=-\frac{m}{r^2}z,
\]
and the ambient coordinates form the complete first nonconstant eigenspace of spherical Brownian motion.
The spherical world is therefore linearly admissible at scale
\(\lambda=m/r^2\).
Its first nonlinear eigenvalue satisfies
\[
\lambda_{\mathrm{nl}}
=
\frac{2(m+1)}{r^2},
\qquad
\frac{\lambda_{\mathrm{nl}}}{\lambda}
=
2+\frac{2}{m}>2,
\]
giving a larger relative linear--nonlinear separation than in the Gaussian world.

\paragraph{Flat toroidal world.}
Consider the equal-radius flat torus
\[
\mathbb T_r^k
=
\left\{
r(\cos\theta_1,\sin\theta_1,\ldots,
\cos\theta_k,\sin\theta_k)
\right\}
\subset\mathbb R^{2k}
\]
with its uniform distribution.
Its ambient radius is constant, and
\[
\Delta_{\mathbb T_r^k}z=-\frac{1}{r^2}z.
\]
The sine and cosine coordinates form the complete first nonconstant Fourier eigenspace, so this world is linearly admissible at scale
\(\lambda=1/r^2\).
For \(k\geq2\), products of first-order modes give
\(\lambda_{\mathrm{nl}}/\lambda=2\); for \(k=1\), it give a ratio of \(4\).
Unequal circle radii produce different coordinate eigenvalues and therefore destroy linear admissibility in the canonical embedding.

\paragraph{Balanced products of spheres.}
Let
\[
\mathcal M
=
\mathbb S^m(r)\times\mathbb S^m(r)
\subset\mathbb R^{2m+2}
\]
with the uniform product distribution.
Both factors have covariance \(r^2/(m+1)\) in their ambient coordinates and coordinate eigenvalue \(m/r^2\).
Their ambient coordinates therefore jointly form the complete first nonconstant eigenspace, making the balanced product linearly admissible.

The Clifford torus used in Section~\ref{sec:morse-scaling},
\[
\mathcal C_d
=
\mathbb S^{d/2}(1)\times\mathbb S^{d/2}(1)
\subset\mathbb R^{d+2},
\]
is an instance of this construction.
Cross-products of first-order modes additionally give
\(\lambda_{\mathrm{nl}}/\lambda=2\).

For a general product \(\mathbb S^{m_1}(r_1)\times\mathbb S^{m_2}(r_2)\), the ambient coordinates form a common first eigenspace when \(\frac{m_1}{r_1^2}=\frac{m_2}{r_2^2}\).
Our isotropy assumption additionally requires \(\frac{r_1^2}{m_1+1}=\frac{r_2^2}{m_2+1}\).
Together, these conditions force \(m_1=m_2\) and \(r_1=r_2\) in the canonical product embedding.

These examples show that linear admissibility is determined jointly by the latent distribution, the embedding geometry, and the positive-pair dynamics.
Gaussian spaces, spheres, equal-radius flat tori, and balanced products of spheres are linearly admissible in their canonical embeddings.
However, the intrinsic manifold alone does not determine admissibility: changing the radii of its embedded factors or the dynamics used to generate positive pairs can prevent the ambient coordinates from forming the complete first nonconstant eigenspace.

\section{Gaussian and spherical spectral-gap comparison}
\label{app:spectral-gap-comparison}

For the Gaussian world,
$\lambda_{\mathrm{nl}}/\lambda=2$, whereas for the uniform sphere,
\[
\frac{\lambda_{\mathrm{nl}}}{\lambda}
=
2+\frac{2}{d-1}.
\]
Figure~\ref{fig:spectral-gap-constant} compares the resulting coefficient in
the approximate-identifiability bound.

\begin{figure}[h]
    \centering
    \includegraphics[width=0.75\columnwidth]{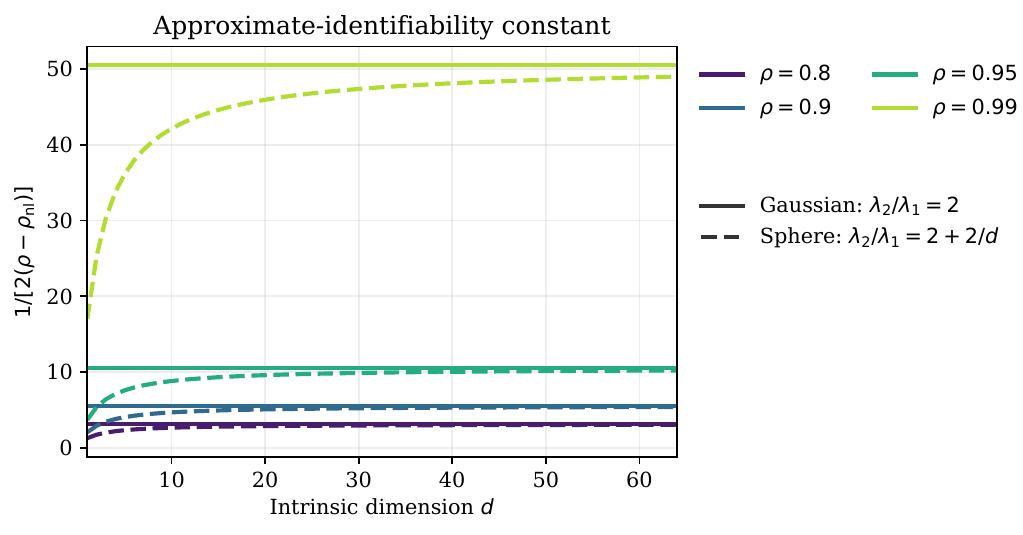}
    \caption{
        Approximate-identifiability coefficient
        $1/[2(\rho-\rho_{\mathrm{nl}})]$ for Gaussian and spherical worlds.
        The spherical advantage is largest in low dimension.
    }
    \label{fig:spectral-gap-constant}
\end{figure}
\section{Proofs of the Main Results}
\label{sec:proofs-main-results}

\subsection{Proof of the converse characterization}

\conversecharacterization*

\begin{proof}
Let $0=\lambda_0<\lambda_1<\lambda_2<\cdots$ denote the eigenvalues of $-D_p$, and let $T_t=e^{tD_p}$. By stationarity, both $Z$ and $Z'$ follow $p$. Hence, for any $h:\mathcal M\to\mathcal M$ satisfying $h_{\#}p=p$,
\begin{align*}
\mathcal L_p(h)
&=\mathbb E\|h(Z)-h(Z')\|^2\\
&=2dc_p-2\sum_{j=1}^d\mathbb E[h_j(Z)h_j(Z')]\\
&=2dc_p-2\sum_{j=1}^d
\langle h_j,T_t h_j\rangle_{L^2(p)}.
\end{align*}
Since $p$ is centered and isotropic, the normalized components
$h_j/\sqrt{c_p}$ form an orthonormal, zero-mean family in $L^2(p)$.
Ky Fan's variational principle~\citep{fan1951maximum} therefore gives
\[
\mathcal L_p(h)
\geq
2dc_p\left(1-e^{-\lambda_1t}\right),
\]
with equality if and only if the components of $h$ span the first non-constant eigenspace $E_1$.

By \eqref{eq:spectral-tightness}, this lower bound is attained. Let $h$ be a distribution-preserving minimizer. Its components therefore span $E_1$, and, by hypothesis, $h$ is linear:
\[
h(z)=Az
\]
for some matrix $A\in\mathbb R^{d\times d}$. Since $h_{\#}p=p$ and
$\operatorname{Cov}_p(Z)=c_pI_d$,
\[
c_pI_d
=
\operatorname{Cov}_p(h(Z))
=
c_pAA^\top,
\]
so $A\in O(d)$. In particular, $A$ is invertible, and the components of $h$ span the same space as the ambient coordinate functions. Consequently,
\[
E_1=\operatorname{span}\{z_1,\ldots,z_d\}.
\]
Since $E_1$ is the eigenspace associated with $\lambda_1$, we obtain
\[
D_pz=-\lambda_1z.
\]

Applying the weighted generator componentwise to the embedding map gives
\[
\Delta_{\mathcal M}z+\nabla_{\mathcal M}\log p
=
-\lambda_1z.
\]
The vector $\Delta_{\mathcal M}z$ is normal to $\mathcal M$, whereas
$\nabla_{\mathcal M}\log p$ is tangent. Decomposing
$z=z^\parallel+z^\perp$ into its tangent and normal components yields
\[
\nabla_{\mathcal M}\log p=-\lambda_1z^\parallel,
\qquad
\Delta_{\mathcal M}z=-\lambda_1z^\perp.
\]
Moreover,
\[
z^\parallel
=
\nabla_{\mathcal M}\left(\frac{\|z\|^2}{2}\right),
\]
and therefore
\[
\nabla_{\mathcal M}
\left(
\log p+\frac{\lambda_1}{2}\|z\|^2
\right)
=0.
\]
Since $\mathcal M$ is connected,
\[
\log p(z)+\frac{\lambda_1}{2}\|z\|^2
\]
is constant on $\mathcal M$. This proves the density and geometric conditions with
$\lambda=\lambda_1$. Together with
\[
E_1=\operatorname{span}\{z_1,\ldots,z_d\},
\]
these are the conditions of linear admissibility at scale $\lambda_1$.
\end{proof}

\subsection{Proof of forward linear identifiability}

\forwardidentifiability*

\begin{proof}
From \eqref{eq:admissible-density},
\[
\nabla_{\mathcal M}\log p=-\lambda z^\parallel.
\]
Combining this identity with \eqref{eq:admissible-geometry} gives
\[
D_pz
=\Delta_{\mathcal M}z+\nabla_{\mathcal M}\log p
=-\lambda z^\perp-\lambda z^\parallel
=-\lambda z.
\]
Thus, the ambient coordinate functions are eigenfunctions of $-D_p$ with eigenvalue $\lambda$. By linear admissibility, they span the complete first non-constant eigenspace:
\[
E_1=\operatorname{span}\{z_1,\ldots,z_d\}.
\]

As in the preceding proof, every distribution-preserving map satisfies
\[
\mathcal L_p(h)
=
2dc_p
-
2\sum_{j=1}^d
\langle h_j,T_t h_j\rangle_{L^2(p)}.
\]
Since the normalized components $h_j/\sqrt{c_p}$ form an orthonormal, zero-mean family, Ky Fan's variational principle yields
\[
\mathcal L_p(h)
\geq
2dc_p\left(1-e^{-\lambda t}\right).
\]
The identity map preserves $p$ and attains this bound. Equality in Ky Fan's principle implies that the components of every minimizer belong to $E_1$. Hence every minimizer has the form
\[
h(z)=Az
\]
for some $A\in\mathbb R^{d\times d}$. Since $h_{\#}p=p$ and
$\operatorname{Cov}_p(Z)=c_pI_d$,
\[
c_pI_d
=
\operatorname{Cov}_p(h(Z))
=
c_pAA^\top,
\]
and therefore $A\in O(d)$.

Moreover, $h(z)=Az$ for $p$-almost every $z$.
Since $p$ has a strictly positive density and $\mathcal M$ is properly embedded, its ambient support is $\mathcal M$.
The identity $A_{\#}p=p$ and the invertibility of $A$ therefore imply $A(\mathcal M)=\mathcal M$.

Conversely, if $Q\in O(d)$ satisfies $Q(\mathcal M)=\mathcal M$, then $Q$ preserves both the induced Riemannian volume and the density $p(z)\propto\exp(-\lambda\|z\|^2/2)$. Thus
$Q_{\#}p=p$, and $h(z)=Qz$ attains the lower bound.
\end{proof}

\subsection{Proof of approximate linear identifiability}

\approximateidentifiability*

\begin{proof}
Since $h_{\#}p=p$ and $p$ is centered, each component of $h$ has zero mean.
Let $Az$ be the componentwise $L^2(p)$-orthogonal projection of $h$ onto $\operatorname{span}\{z_1,\ldots,z_d\}$, and write
\[
h=Az+r,
\qquad
r\perp\operatorname{span}\{1,z_1,\ldots,z_d\},
\]
for some $A\in\mathbb R^{d\times d}$.
Orthogonality and the spectral decomposition of $T_t$ imply
\begin{align*}
\sum_{j=1}^d \langle h_j,T_t h_j\rangle_{L^2(p)}
&\leq
\rho\|Az\|_{L^2(p)}^2 +\rho_{\mathrm{nl}}\|r\|_{L^2(p)}^2,\\
dc_p
&=
\|Az\|_{L^2(p)}^2
+\|r\|_{L^2(p)}^2.
\end{align*}
Using the stationary expansion of the alignment loss therefore gives
\[
\mathcal L_p(h)
\geq 2dc_p(1-\rho)
+2(\rho-\rho_{\mathrm{nl}})\|r\|_{L^2(p)}^2.
\]
Combining this inequality with \eqref{eq:approximate-alignment-assumption} yields
\[
\|r\|_{L^2(p)}^2
\leq\frac{\varepsilon}{2(\rho-\rho_{\mathrm{nl}})}=\eta,
\]
which is \eqref{eq:approximate-linear-bound}.
Equality holds in the underlying spectral inequality when the nonlinear remainder lies entirely in the eigenspace associated with $\lambda_{\mathrm{nl}}$.
Thus, the constant is sharp at the spectral level, without asserting feasibility under the additional manifold-valued and distribution-preserving constraints.

It remains to control the linear ambiguity. Since \(h_{\#}p=p\) and \(p\) is centered and isotropic, the components of \(h\) have Gram matrix \(c_p I_d\).
Moreover, \(r\) is orthogonal to the coordinate functions. Therefore,
\[
c_p I_d = c_p A A^\top + R,
\qquad
R_{ij}=\langle r_i,r_j\rangle_{L^2(p)}.
\]
\[
c_pI_d=c_pAA^\top+R,
\qquad
R_{ij}=\langle r_i,r_j\rangle_{L^2(p)}.
\]
The matrix $R$ is positive semidefinite and
$\|R\|_F\leq\operatorname{tr}(R)=\|r\|_{L^2(p)}^2\leq\eta$. Therefore
\[
\|AA^\top-I_d\|_F\leq\frac{\eta}{c_p}.
\]
Let $A=QS$ be a polar decomposition, with $Q\in O(d)$ and $S\succeq0$. Since
$|s-1|\leq|s^2-1|$ for every singular value $s\geq0$,
\[
\|A-Q\|_F\leq\|AA^\top-I_d\|_F\leq\frac{\eta}{c_p}.
\]
Finally, the linear and nonlinear components are orthogonal, so
\[
\|h-Qz\|_{L^2(p)}^2
=\|(A-Q)z\|_{L^2(p)}^2+\|r\|_{L^2(p)}^2
=c_p\|A-Q\|_F^2+\|r\|_{L^2(p)}^2
\leq\frac{\eta^2}{c_p}+\eta.
\]
This proves \eqref{eq:approximate-orthogonal-bound}.
\end{proof}

\subsection{Proof of the radial obstruction and spherical rigidity}

\radialobstruction*

\begin{proof}
The density condition gives
$\nabla_{\mathcal M}\log p=-\lambda z^\parallel$. Moreover,
\[
\Delta_{\mathcal M}\|z\|^2
=
2m+2\langle z,\Delta_{\mathcal M}z\rangle
=
2m-2\lambda\|z^\perp\|^2,
\qquad
\nabla_{\mathcal M}\|z\|^2=2z^\parallel.
\]
Therefore,
\[
D_p\|z\|^2
=
2m-2\lambda\|z^\perp\|^2
-2\lambda\|z^\parallel\|^2
=
2m-2\lambda\|z\|^2.
\]
It follows that $r(z)=\|z\|^2-m/\lambda$ satisfies
\[
-D_pr=2\lambda r.
\]

If $r$ is nonzero, then it is orthogonal to both the constant eigenspace and the linear eigenspace, since its eigenvalue $2\lambda$ differs from $0$ and $\lambda$. Thus, $r$ is a nonlinear eigenfunction and
$\lambda_{\mathrm{nl}}\leq2\lambda$.

Otherwise, $r$ vanishes identically, so
\[
\mathcal M
\subseteq
\mathbb S^{d-1}\left(\sqrt{\frac{m}{\lambda}}\right).
\]
In particular, $\lambda_{\mathrm{nl}}>2\lambda$ forces this constant-radius alternative.

If moreover $m=d-1$ and $\mathcal M$ is closed, its inclusion into
$\mathbb S^{d-1}(\sqrt{m/\lambda})$ is a local diffeomorphism. Its image is therefore both open and closed. Since the sphere is connected, the image is the entire sphere, and hence
\[
\mathcal M
=
\mathbb S^{d-1}\left(\sqrt{\frac{d-1}{\lambda}}\right).
\]
\end{proof}

%\begin{table}[t]
%\centering
%\small
%\begin{tabular}{lccccc}
%\toprule
%Geometry
%& Intrinsic dim.
%& Embedding dim. \(d_{\mathrm{eff}}\)
%& Relative gap
%& Invariant loss
%& Regularization loss (mmd) \\
%\midrule
%Gaussian \(\mathbb R^m\)
%& \(m\)
%& \(m\)
%& \(2\)
%& \(\mathcal O(mB)\)
%& \(\mathcal O\!\left(m(B^2+CB)\right)\)
%\\
%Sphere \(\mathbb S^m\)
%& \(m\)
%& \(m+1\)
%& \(2+\frac{2}{m}\)
%& \(\mathcal O((m+1)B)\)
%& \(\mathcal O((m+1)B^2)\)
%\\
%Veronese \(\mathbb{RP}^m\)
%& \(m\)
%& \(d_{\mathrm{Ver}}=\frac{(m+1)(m+2)}{2}-1\)
%& \(2+\frac{4}{m+1}\)
%& \(\mathcal O(d_{\mathrm{Ver}}B)\)
%& \(\mathcal O(d_{\mathrm{Ver}}B^2)\)
%\\
%\bottomrule
%\end{tabular}
%\caption{
%Spectral and computational comparison at fixed intrinsic dimension \(m\).
%The Veronese embedding slightly improves the relative spectral gap, but requires
%\(d_{\mathrm{Ver}}=\frac{(m+1)(m+2)}{2}-1=\mathcal O(m^2)\) ambient dimensions.
%}
%\label{tab:geometry_gap_cost}
%\end{table}

\end{document}